\documentclass[journal]{IEEEtran}
\usepackage{graphicx}
\usepackage{amsmath}
\usepackage{amsthm}
\usepackage{algorithm}
\usepackage{algpseudocode}
\usepackage{algcompatible}
\usepackage{mathtools}
\usepackage{amsfonts}
\usepackage{comment}
\usepackage{color}
\usepackage{soul}
\usepackage{subcaption}
\usepackage{url}

\DeclarePairedDelimiter\ceil{\lceil}{\rceil}

\ifCLASSINFOpdf
\else
\fi

\newtheorem{prop}{Proposition}
\begin{document}

\title{Towards Trusted and Fail Safe Execution of Cloud Generated Machine Learning Models}

\title{Intelligent Selection of Trusted Machine Learning Models in Support of IoT and Smart City Services}

\title{Trust-Based Cloud Machine Learning Model Selection For Industrial IoT and Smart City Services}

\author{
}

\author{Basheer~Qolomany,~\IEEEmembership{Graduate Student Member,~IEEE,}
        Ihab~Mohammed,~\IEEEmembership{Graduate  Student Member,~IEEE},
        Ala~Al-Fuqaha,~\IEEEmembership{Senior Member,~IEEE},
        Mohsen~Guizani,~\IEEEmembership{Fellow,~IEEE},
        Junaid~Qadir,~\IEEEmembership{Senior Member,~IEEE}
\thanks{B. Qolomany is with  Department of Cyber Systems, College of Business \& Technology, University of Nebraska at Kearney, Kearney, NE 68849, USA (e-mail: qolomanyb@unk.edu) }       
\thanks{I. Mohammed is with the Department of Computer
Science, Western Michigan University, Kalamazoo, MI 49008 USA (e-mail: ihabahmedmoha.mohammed@wmich.edu).
}
\thanks{A. Al-Fuqaha is with the Information and Computing Technology Division, College of Science and Engineering, Hamad Bin Khalifa University, Doha, Qatar and with the Computer Science Department, Western Michigan University, Kalamazoo, Michigan (e-mail: 
ala@ieee.org).}
\thanks{M. Guizani is with the Computer Science and Engineering Department, Qatar University, Doha, Qatar  (e-mail: 
mguizani@ieee.org).}

\thanks{J. Qadir is with Information Technology University, Lahore, Pakistan  (e-mail: 
junaid.qadir@itu.edu.pk).}
        }

\maketitle

\begin{abstract}

%The global market for Machine Learning (ML) has grown rapidly over the last few years largely due to the fast pace of integrating ML with many facets of everyday life, particularly for enabling smart Internet-of-Things (IoT) services. 

With Machine Learning (ML) services now used in a number of mission-critical human-facing domains, ensuring the integrity and trustworthiness of ML models becomes all-important. In this work, we consider the paradigm where cloud service providers collect big data from resource-constrained devices for building ML-based prediction models that are then sent back to be run locally on the intermittently-connected resource-constrained devices. Our proposed solution comprises an intelligent polynomial-time heuristic that maximizes the level of trust of ML models by selecting and switching between a subset of the ML models from a superset of models in order to maximize the trustworthiness while respecting the given reconfiguration budget/rate and reducing the cloud communication overhead. We evaluate the performance of our proposed heuristic using two case studies. First, we consider \textit{Industrial IoT (IIoT)} services, and as a proxy for this setting we use the turbofan engine degradation simulation dataset to predict the remaining useful life of an engine. Our results in this setting show that the trust level of the selected models is 0.49\% to 3.17\% less compared to the results obtained using Integer Linear Programming (ILP). Second, we consider \textit{Smart Cities} services, and as a proxy of this setting we use an experimental transportation dataset to predict the number of cars. Our results show that the selected model's trust level is 0.7\% to 2.53\% less compared to the results obtained using ILP. We also show that our proposed heuristic achieves an optimal competitive ratio in a polynomial-time approximation scheme for the problem.

\end{abstract}

\IEEEpeerreviewmaketitle

\section{Introduction}

%\subsection{Motivation of Trusted ML models}
The global market for Machine Learning (ML) has grown rapidly over the last few years largely due to the fast pace of integrating ML with many facets of everyday life, particularly for enabling smart Internet-of-Things (IoT) services. Most of today$'$s IoT predictive analytics rely on cloud-based services, in which IoT resource-constrained devices continuously send their collected data to the cloud \cite{ray_survey_2016}. Resource-constrained devices have limited processing, communication and/or storage capabilities, and often run on batteries. On the cloud, ML as a Service (MLaaS) providers carry out the prediction process and provide data pre-processing, model training, model evaluation, and model update capabilities \cite{tramer_stealing_2016}. The MLaaS market is expected to exceed \$3,754 million by 2024 at a compound annual growth rate (CAGR) of 42\% in the given forecast period \cite{noauthor_machine_2017}. Typical systems include electrical power grids \cite{dan_cloud_2013}, intelligent transportation and vehicular management \cite{meneguette_vehicular_2016}, health care devices \cite{hanen_enhanced_2016}, household appliances \cite{zhang_iehouse_2017}, predictive maintenance systems \cite{mourtzis_cloud-based_2016} in Industrial IoT (IIoT) and many more. However, ML models can be targeted by malicious adversaries \cite{barreno_can_2006} due to the participatory nature of such systems. Cyber-attacks against critical infrastructure are not just theories, they are very real and have already been used to effect. For example, in December 2015 a cyber-attack on Ukraine's power grid left 700,000 people without electricity for several hours \cite{lee_analysis_2016}. The Stuxnet worm, which was first uncovered in 2010, is believed to be responsible for causing substantial damage to Iran's nuclear program \cite{kushner_real_2013}. In March 2016, the U.S. Justice Department indicated that cyber-attacks tied to the Iranian regime \cite{noauthor_seven_2016} targeted 46 major financial institutions and a dam outside of New York City.  

%Due to the market competition among various cloud service providers, their MLaaS services tend to have similar offerings. 

Perhaps the most pressing challenge in the emerging cloud computing area is that of establishing trust \cite{khan_establishing_2010} \cite{pearson_privacy_2013}. Despite the importance of trust in cloud computing, a common conceptual model of trust in cloud computing has not yet been defined \cite{chiregi_cloud_2017} and it is becoming increasingly complex for cloud users to distinguish between service providers offering similar types of services in terms of trustworthiness \cite{sidhu_compliance_2014}. Trust has been investigated from different disciplinary lenses such as psychology, sociology, economics, management, and information systems (IS) but there is no commonly accepted definition of trust \cite{mcknight_what_nodate} \cite{mcknight_what_2001}. That is, depending on the context, we may think of many different things when someone uses the word `trust.' Merriam-Webster's dictionary defines the word 'trust' as \emph{"assured reliance on the character, ability, strength, or truth of someone or something."}  Our definition for the trust in this paper refers to the ML models that agree most with the predictions of an ensemble of ML models. Therefore, a model that agrees more with the predictions of the ensemble is more `trustworthy' compared to the one that agrees less with the ensemble. For example, assume that we have 5 models ($M_1$, $M_2$, $M_3$, $M_4$ and $M_5$), and the model $M_1$ agrees with 3 other models, while $M_2$ agrees with only 2 other models, then $M_1$ is more trustworthy than $M_2$.

The performance of ML models can be quantified based on their decision time, accuracy, and precision of resulting decisions \cite{domingos_few_2012}. However, as such models are used for more critical and sensitive decisions (e.g., whether a drug should be administered to a patient or should an autonomous vehicle stop for pedestrians), it becomes more important to ensure that they provide high accuracy and precision guarantees. Assessing learning models in terms of trustworthiness along with the traditional criteria of decision time, accuracy, and precision establishes a tradeoff between simplicity and power \cite{kaul_speed_2018}. ML classifiers are vulnerable to adversarial examples, which are samples of input data that are maliciously modified in a way that is intended to cause an ML classifier to misclassify similar examples. Moreover, it is known that adversarial examples generated by one classifier are likely to cause another classifier to make the same mistake \cite{kurakin_adversarial_2016}. In many cases, the modifications can successfully cause a classifier to make a mistake even though the modifications is imperceptible to a human observer. In general, adversarial attacks can be classified into a misclassification attack or a targeted attack based on its goals \cite{elsayed_adversarial_2018} \cite{hosseini_blocking_2017} \cite{moati_reputation-based_2014}. In a \textit{misclassification attack}, the adversary intends to cause the classifier to output a label different from the original label. In a \textit{targeted attack}, on the other hand, the adversary intends to cause the classifier to output a specific misleading label. 

In this paper, we envision the paradigm where resource-constrained IoT devices execute ML models locally, without necessarily being always connected to the cloud. Some advantages of our proposed heuristic is its applicability to a number of applications scenarios beyond the pale of the traditional paradigm where it is not desirable to execute the model on the cloud due to latency, connectivity, energy, privacy, and security concerns. Consequently, it is expected that the users should be able to determine the trustworthiness of service providers in order to select them with confidence and with some degree of assurance that their service offerings will not behave unpredictably or maliciously. Our proposed heuristic strives to minimize the communications overhead between the cloud and the resource-constrained devices. Selected ML models are sent to resource-constrained devices to be used.  The proposed heuristic also has a limitation that it is not intended to improve the trustworthiness of the models trained in Federated Learning (FL) systems when each client preserves its own data locally. Instead, our approach can be applied to improve the trustworthiness of centralized approach of learning, when all the clients send their data to  a MLaaS provider to build ML model on the cloud, then this model will be sent back to be hosted on a resource-constrained devices.  The target of the proposed heuristic is not to handle all different types of the attacks. We only consider poisoning attacks on ML classifiers. Within this scenario, an attacker may poison the training data by injecting carefully designed samples to eventually compromise the whole learning process. Figure \ref{Fig_1.png} shows a general architecture for the proposed system. On the cloud side we have $M$ ML models, $model_{1}$, $model_{2}$, $\dots$, $model_{M}$.

%In general, attacks based on adversarial examples based on its goals can be classified into \cite{elsayed_adversarial_2018} \cite{hosseini_blocking_2017}: 1) \textit{Misclassification Attack}: The adversary causes the classifier to output a label different from the original label; and 2) \textit{Targeted Attack}: The adversary causes the classifier to output a specific misleading label. 

%\subsection{Contribution}

The main contributions of the work can be summarized as:

\begin{itemize}
\item [(i)] We formulate the problem of finding a subset of ML models that maximize the trustworthiness while respecting given reconfiguration budget and rate constraints. We also prove the problem is NP-complete. 
\item [(ii)] We propose a heuristic that maximizes the level of trust of ML models and finds a near-optimal solution in polynomial time by selecting a subset from a superset of ML models. Our trust metric of an ML model is based on recent and past historical data that measure the degree of agreement of the ML model with other models in an ensemble of ML models. 
\item [(iii)] The proposed system has fail-safe state, such that if the proposed heuristic does not find a trusted ML model in the superset of models, it sends a fail-safe execution alert. This alert informs the resource-constrained devices that no trusted ML model exists in the system. As a result, the resource-constrained devices can fail safely as required by the application that they service.
\item [(iv)] Building on the above insights, we apply the proposed heuristic to two different training datasets. The first dataset, based on the CityPulse project \cite{noauthor_citypulse_nodate}, is used to predict the number of vehicles as a surrogate use-case for smart city services. The second data set, provided by the Prognostics CoE at NASA Ames \cite{saxena_turbofan_2008}, is used to predict the remaining useful life of a turbofan engine as a surrogate use-case for IIoT smart services.
\item [(v)] 	We applied the swap $x$ and $100-x$ percentiles approach as a causative adversarial attack by altering the training dataset label as we will describe in Section VI.

\end{itemize}

For the convenience of the readers, Table \ref{Table_1_} provides a list of the acronyms used in this paper.
\begin{table}[h]
\centering
\scriptsize
\caption{List of Important Acronyms  Used}
\label{Table_1_}
\begin{tabular}{c p{4.7cm}} 
 \hline
% Abbreviations & \\ [0.5ex] 
%  \hline
DL	& Deep Learning \\
FL	& Federated Learning \\
IIoT & Industrial Internet of Things \\
ILP & Integer Linear Programming \\
IoT & Internet of Things \\
IS & Information System \\
LP & Linear Programming  \\
LSTM & Long Short-Term Memory \\
ML & Machine Learning \\
MLaaS & Machine Learning as a Service \\
NP & Non-deterministic Polynomial-time\\
PM	& Predictive Maintenance  \\
RMSE & Root Mean Square Error \\
TML & Trusted Machine Learning \\
\hline
\end{tabular}
\end{table}

The remainder of this paper is organized as follows:
Section \ref{RelatedWorkSec} presents the most recent related work. The background information of case studies and thread model is introduced in Section \ref{BackgroundSec}. Section \ref{SystemModelSec} discusses the system model and problem formulation. Section \ref{ProposedSolSec} discusses our proposed solutions and their competitive ratio analysis. Section \ref{PerformanceEvaSec} presents our  experimental setup, experimental results and the lessons learned. Finally, Section \ref{ConclusionSec} concludes the paper and discusses future research directions. 

\begin{figure} 
\centering
\includegraphics[width=0.50\textwidth]{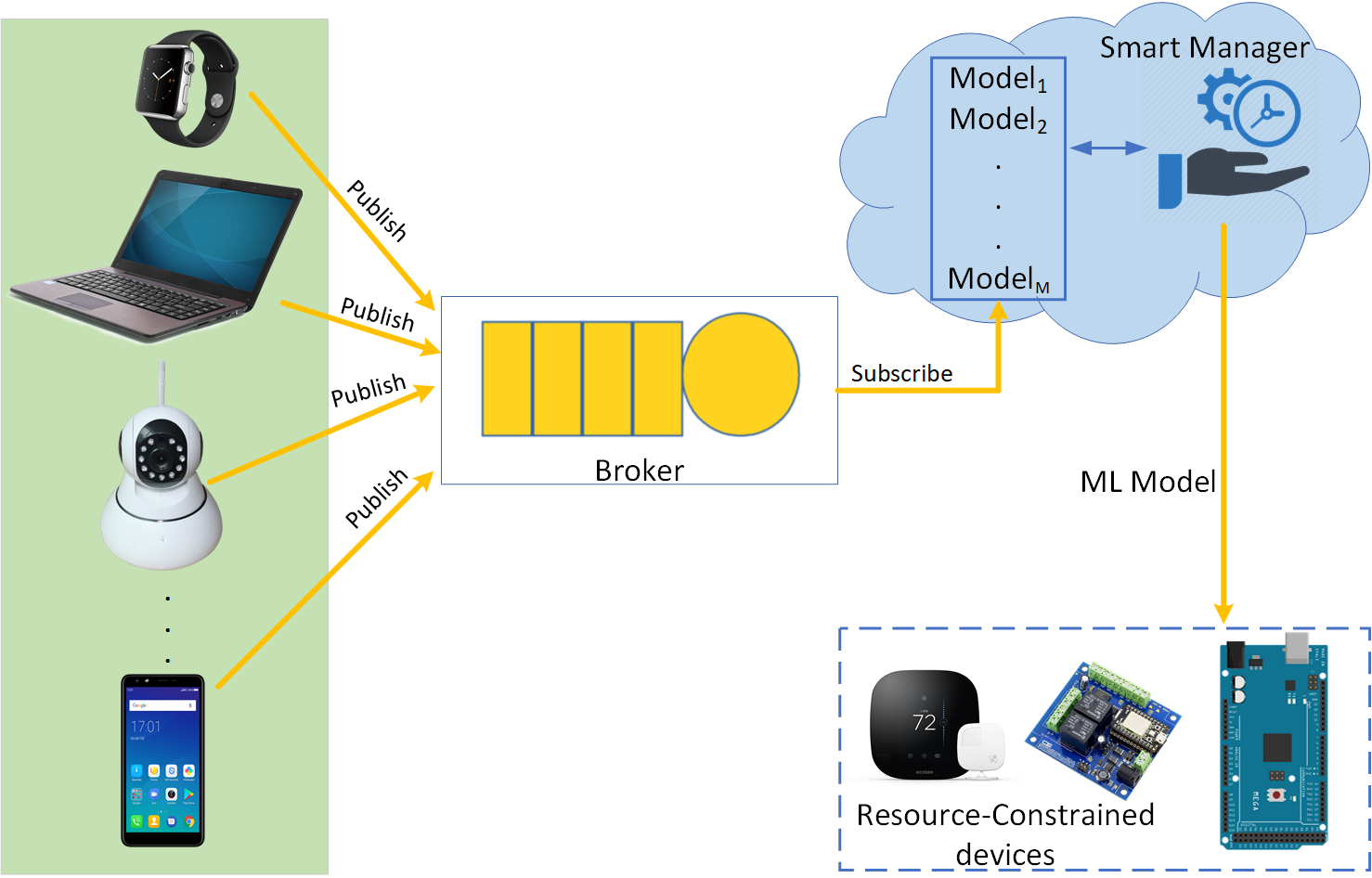}
\caption{General architecture for the proposed system of selecting a trustworthy subset of ML models built by different cloud service providers.}
\label{Fig_1.png}
\end{figure} 

\section{Background}
\label{BackgroundSec}
\subsection{Case Studies}

Several case studies could be considered, in which the proposed heuristic helps to gain the best trust level. Here, we discuss two representative case studies. The first case study considers IIoT predictive maintenance while the second one considers real-time traffic flow prediction in smart cities. Figure \ref{Fig2_1.png} illustrates the trust-based model selection problem addressed in this research and also depicts the two considered case studies. During each decision period, our proposed heuristic switches between the subset of selected models with a goal of maximizing the overall trustworthiness while respecting the given reconfiguration budget and rate.

\begin{figure}[h]
\centering
\includegraphics[width=0.50\textwidth]{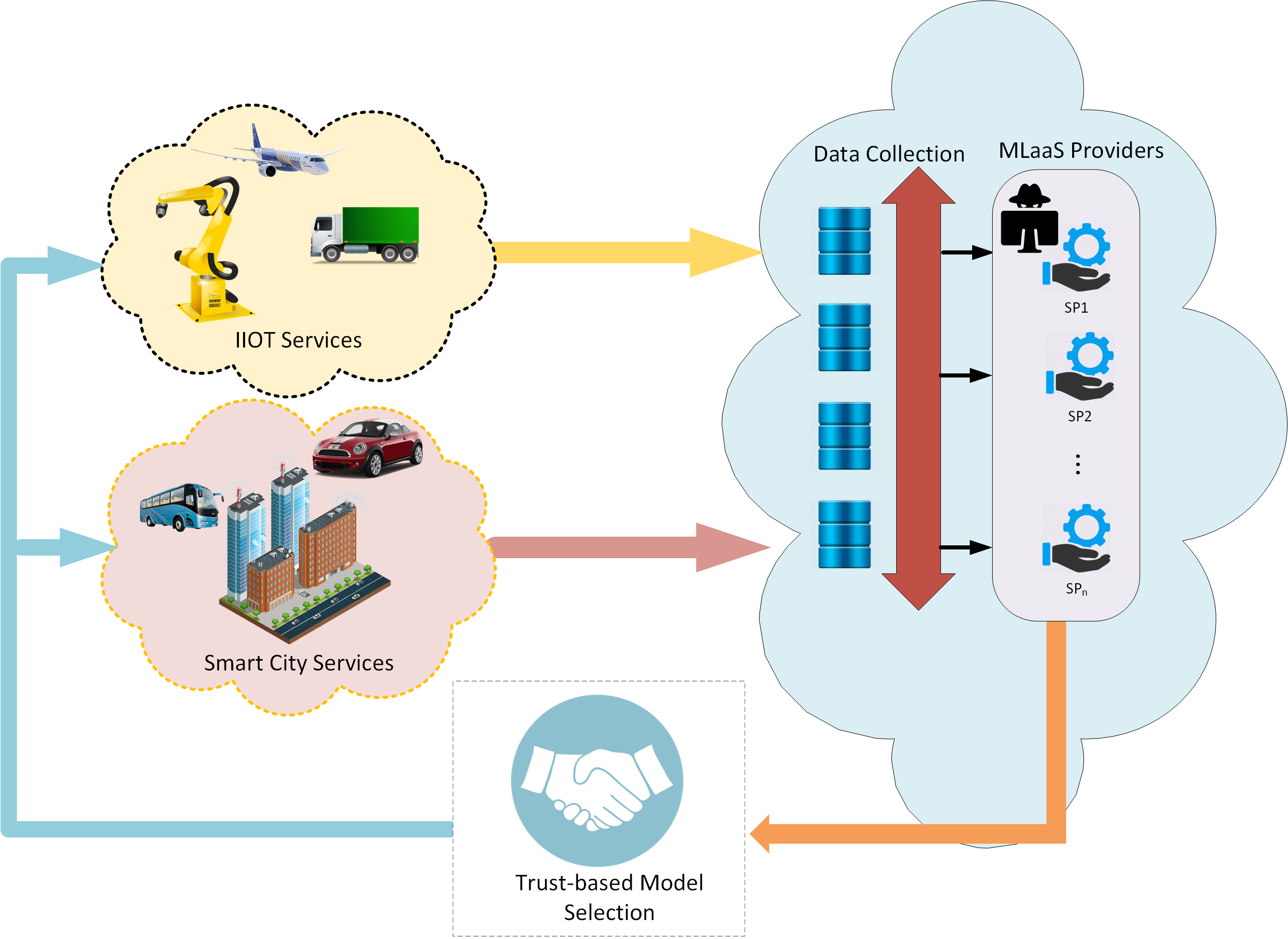}
\caption{Trust-based model selection problem for IIoT and smart city case scenarios.}
\label{Fig2_1.png}
\end{figure}

\subsubsection{IIoT Predictive Maintenance Case Study}

A predictive maintenance (PM) strategy uses ML methods to identify, monitor, and analyze system variables during operation. Also, PM alerts operators to preemptively perform maintenance before a system failure occurs \cite{cline_predictive_2017}. Being able to stay ahead of equipment shutdowns in a mine, steel mill, or factory, PM can save money and time for a busy enterprise \cite{sipos_log-based_2014}. With PM, the data is collected over time to monitor the state of the machine and is then analyzed to find patterns that can help predict failures. In many cases, it is desirable to have prediction models hosted on resource-constrained embedded devices. Predictive maintenance systems need to provide real-time control of the machines based on the deviation of the real-time flow readings from the predicted ones. In such systems, embedded sensors collect short-term state of the machine readings, which are relayed to the cloud directly through communications infrastructure, or indirectly through the use of ferry nodes. Because of its compute and store capabilities, the cloud is capable of collecting the short-term readings to build long-term big data of sensor readings. These readings are then utilized to build a PM model for each of the underlying flow sensors. The constructed models are then sent back to the flow sensors so that they actuate their associated machines when a deviation is observed between the actual and projected flow readings. 

There are scenarios in which cyber-attackers attempt to compromise PM models directly. Consequently, data that leaves its internal operating environment is subject to third-party attacks. For instance, an adversary can create a causative attack to poison the learner's classifications. This is possible by altering the training process through influence over the training data. Therefore, when the system is re-trained, the learner learns an incorrect decision-making function. Thus it is important to ensure the trustworthiness of ML models before they are hosted and used on resource-constrained devices.

\subsubsection{Smart City (Traffic Flow Prediction) Case Study}

Traffic flow prediction plays an important role in intelligent transportation management and route guidance. Such predictions can help in relieving traffic congestion, reducing air pollution, and in providing secure traffic conditions \cite{lv_traffic_2015}. Traffic flow prediction heavily depends on historical and real-time traffic data collected from various sensor sources. These sources include inductive loops, radars, cameras, mobile global positioning systems (GPS), crowdsourcing, social media, etc. Transportation management and control are now becoming more data-driven \cite{paul_chapter_2017}. However, inferring traffic flow under real-world conditions in real-time is still a challenging research problem due to the computational complexity of building, training, learning and storing traffic flow models on resource-constrained devices. 
 
In our proposed approach, various sensor technologies are used to automatically collect  short-term data of the traffic flow and send them to the cloud through communications infrastructure or through the use of ferry nodes. The cloud is capable of collecting the short-term readings to build long-term big data of sensor readings. These readings are then utilized by MLaaS service providers to build a model. The constructed models are then sent back to be hosted on the resource-constrained devices, in order to predict the traffic flow in real-time. Intelligent transportation systems are highly visible, and attacks against them result in high impact on critical infrastructure. For instance, the attacks can cause vehicular accidents or create traffic jams that affect freight movements, and daily commutes, etc. Thus to make the traffic movement more efficient and improve road safety, road operators need to constantly monitor traffic and current roadway conditions by using an array of cameras and sensors that are strategically placed on the road network. These cameras and sensors send back real-time data to the control center \cite{huq_cyberattacks_2017}. The data is subject to causative adversarial attacks, which are launched by altering the training process by influencing the training data and consequently causing the learner to learn an incorrect decision-making function.
 
 Figure \ref{Fig_2.png} illustrates the use of message ferries to collect data from resource-constrained devices. Collected data is delivered to the cloud in order to build the ML models by the MLaaS service providers. Next, the ferrying nodes deliver the ML models to be hosted on the resource-constrained devices.

\begin{figure} 
\centering
\includegraphics[width=0.50\textwidth]{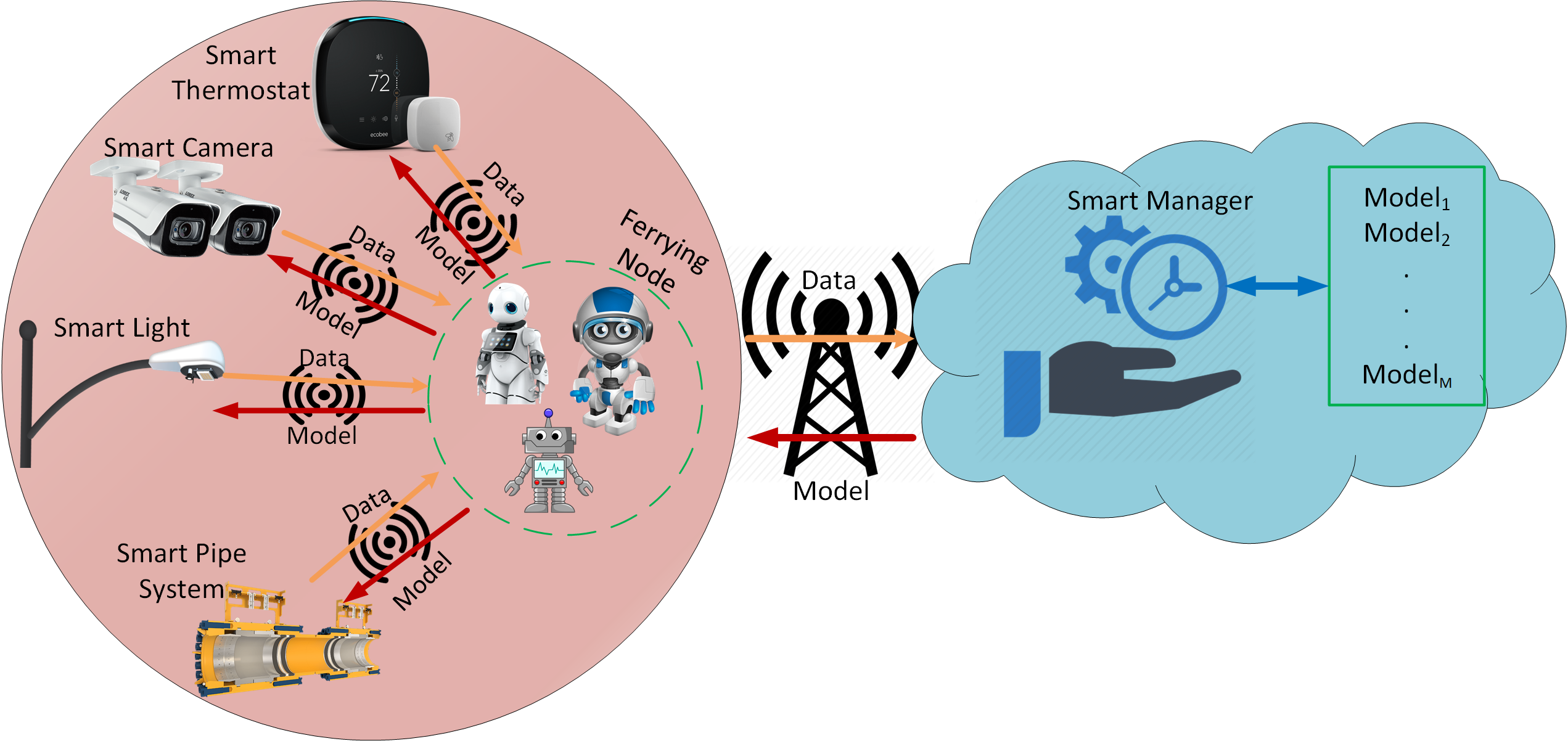}
\caption{The exchange of data/models between resource-constrained devices and the cloud using message ferries.}
\label{Fig_2.png}
\end{figure}

\subsection{Threat Model}

\subsubsection{Adversary Knowledge}
For both the case studies, we only consider poisoning attacks on the ML classifiers. Within this scenario, an attacker may poison the training data by injecting carefully designed samples to eventually compromise the whole learning process. Poisoning may thus be regarded as an adversarial contamination of the training data.  In our experiments, we use the swap $x$ and $100-x$ percentiles attack model as a causative attack against the LSTM algorithm.

\subsubsection{Adversary Goal}
The goal of an adversarial MLaaS provider is to deliver an ML model that results in sub-optimal or erroneous results when executed on resource-constrained IoT devices. The incentive of the adversarial MLaaS provider is to seek gains by colluding with business competitors of MLaaS clients.

\section{Related Work}
\label{RelatedWorkSec}
In this section, we review recent related works. Figure \ref{Fig_3_2.png} shows the research gap that we address in this research. To the best of our knowledge, this paper is the first attempt at designing an intelligent polynomial-time heuristic on the cloud that selects the ML models that should be hosted on IoT resource-constrained devices in order to maximize the trustworthiness of the overall system.

\begin{figure}[h]
\centering
\includegraphics[width=0.50\textwidth]{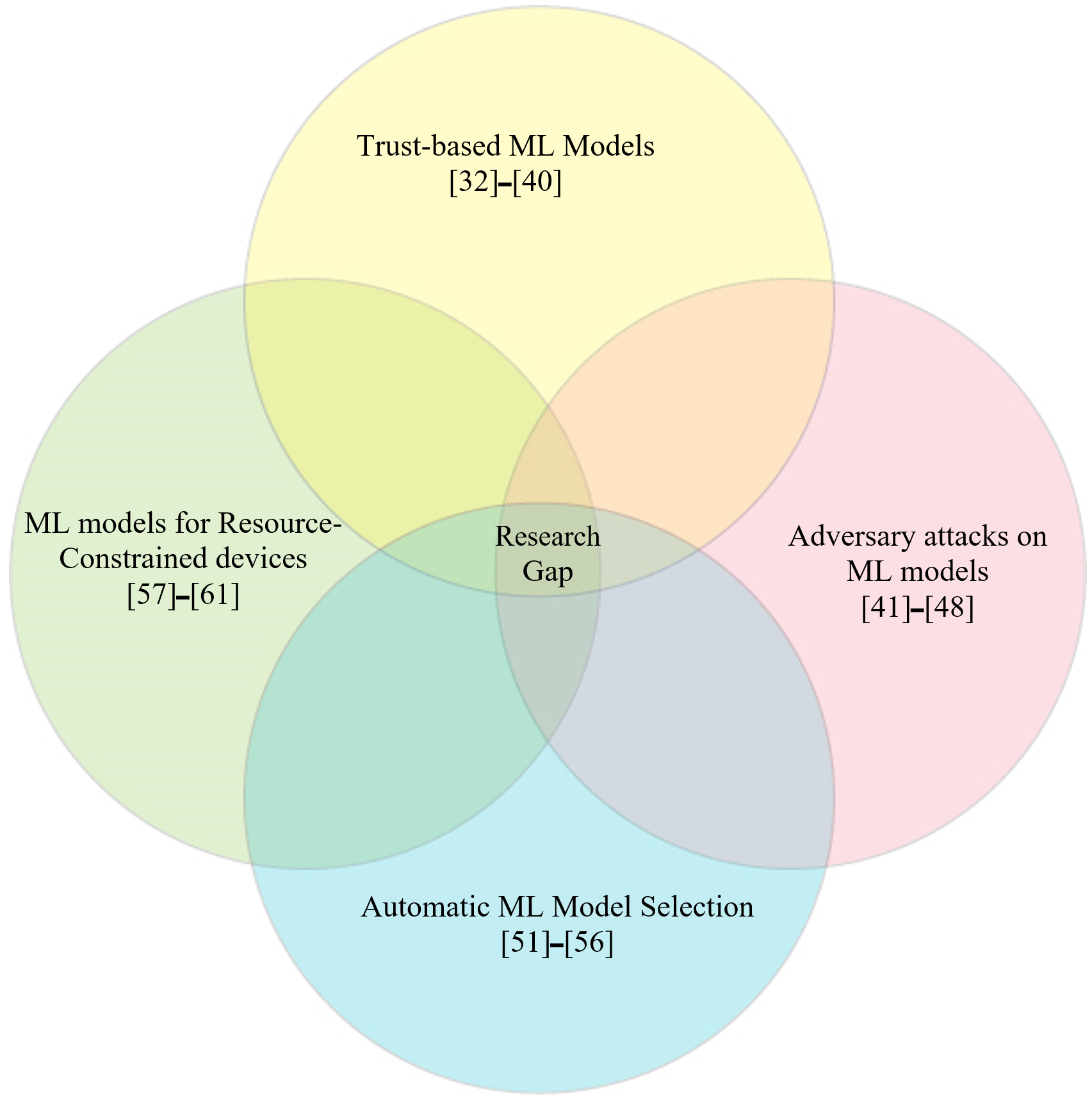}
\caption{The research gap addressed in the paper.}
\label{Fig_3_2.png}
\end{figure}

\subsection{Trust-based ML Models}

Researchers have proposed various approaches to design machine learning algorithms that are trustworthy when using predictions to make critical decisions in real-world applications including healthcare, law, self-driving cars etc. Speicher \textit{et al.} \cite{speicher_reliable_2017} propose an approach to establish of complex ML models by ensuring that in a particular way, a complex model to achieve correct predictions at least on all those data points where a trusted model was already correct. 

Ghosh \textit{et al.} \cite{ghosh_trusted_nodate} proposed the Trusted ML (TML) framework for self-driving cars that uses principles from formal methods for learning ML models. These ML models satisfy properties in temporal logic by using model repair or the data from which the model is learned. Zhang \textit{et al.} \cite{zhang_training_nodate} propose Debugging Using Trusted Items (DUTI) algorithm that uses trusted items to detect outlier and systematic training set bugs. The approach looks for the smallest set of changes in the training set labels, such that, the model learned from this corrected training set predicts labels of the trusted items correctly. 

Ribeiro \textit{et al.} \cite{ribeiro_why_2016} proposed the LIME algorithm, which explains the predictions of any classifier or regressor in an interpretable manner by approximating an interpretable model locally around the prediction. The authors also proposed a method called SP-LIME to select representative and non-redundant predictions, which provides a global view of the model to users. The authors applied the proposed algorithm on both simulated and human subjects in order to decide between and assess models and also identified reasons for not trusting a classifier. Jayasinghe \textit{et al.} \cite{jayasinghe_machine_2019} proposed trust assessment model which specifies the formation of trust from raw data to a final trust value, they proposed an algorithm based on machine learning principles that determine whether an incoming interaction is trustworthy, based on several trust features corresponding to an IoT environment.  Fariha \textit{et al.} \cite{fariha_data_2020} introduced data invariant technique as an approach to achieve  trusted machine learning by reliably detecting tuples on which the prediction of a machine-learned model should not be trusted. They proposed a quantitative semantics to measure the degree of violation of a data invariant, and establish that strong data invariants can be constructed from observations with low variance on the given dataset. Drozdal \textit{et al.} \cite{drozdal_trust_2020} explore trust in the relationship between human data scientists and models produced by AutoML systems. They find that including transparency features in an AutoML tool increased user trust and understandability in the tool; and out of all proposed features, model performance metrics and visualizations are the most important information to data scientists when establishing their trust with an AutoML tool.

Wahab \textit{et al.} \cite{wahab_optimal_2020} proposed a solution for maximizing the detection of VM-based DDoS attacks in cloud systems. Their proposed solution has two components. First, they proposed a trust model between the hypervisor and its guest VMs for the purpose of establishing a credible trust relationship between the hypervisor and guest VMs. Second, they designed a trust-based maximin game between DDoS attackers and hypervisor to minimize the cloud system's detection and maximize this minimization under limited budget of resources. In \cite{liu_vision-based_2008}, the authors make three arguments about the trustworthiness of deep learning (DL) systems to prevent the deception of the algorithm: (1) the trustworthiness should be an essential and mandatory component of a DL system for algorithmic decision making; (2) the trust of a DL model should be evaluated along multiple dimensions in terms of its correctness, accountability, transparency, and resilience; and (3) there should be a proactive safeguard mechanisms to enforce the trustworthiness of a deep learning framework.

In this work, the trust metric of an ML model is based on recent and past historical data that measure the degree of agreement of the ML model with other models in an ensemble of ML models.

\subsection{Adversary attacks on ML models}
Recent research shows that ML models trained entirely on private data are still vulnerable to adversarial examples, which have been maliciously altered so as to be misclassified by a target model while appearing unaltered to the human eye \cite{kurakin_adversarial_2016}\cite{elsayed_adversarial_2018}.
Madry \textit{et al.} \cite{madry_towards_2019} propose an approach to study the adversarial robustness of neural networks through the lens of robust optimization, this approach enables to identify methods for both training and attacking neural networks models.  
Finlayson \textit{et al.} \cite{finlayson_adversarial_2018} demonstrate that adversarial examples are capable of manipulating deep learning systems. They synthesize a body of knowledge about the healthcare system across three clinical domains to argue that medicine may be uniquely susceptible to adversarial attacks. 
Huang \textit{et al.} \cite{huang_adversarial_2011} discuss the effective machine learning techniques against an adversarial opponent. They introduce two machine learning models for modeling an adversary’s capabilities and discuss how specific application domain, features and data distribution restrict an adversary’s attacks.
Saadatpanah \textit{et al.} \cite{saadatpanah_adversarial_2019} discuss how the machine learning methods in industrial copyright detection systems are susceptible to adversarial attacks and why those methods are particularly vulnerable to attacks.
Ren \textit{et al.} \cite{ren_adversarial_2020} introduce the theoretical foundations, algorithms, and applications of adversarial attack and defense techniques in deep learning models. 
Chakraborty \textit{et al.} \cite{chakraborty_adversarial_2018} provide a discussion on different types of adversarial attacks with various threat models and also elaborate the efficiency and challenges of recent countermeasures against them.
Akhtar  and Mian  \cite{akhtar_threat_2018} present a comprehensive survey paper on adversarial attacks on deep learning in computer vision. 
Yuan \textit{et al.} \cite{yuan_adversarial_2019} investigate and summarize the approaches for generating adversarial examples, applications for adversarial examples, and corresponding countermeasures for deep neural network models.

\subsection{Automatic Model Selection}

Researchers have proposed various automatic selection methods for ML algorithms. ML model selection is the problem of determining which algorithm, among a set of ML algorithms, is the best suited to the data \cite{forster_key_2000}. Choosing the right technique is a crucial task that directly impacts the quality of predictions. However, deciding which ML technique is well suited for processing specific data is not an easy task, even for an expert, as the number of choices is usually very large \cite{kotthoff_algorithm_2016}. 

Auto-WEKA \cite{thornton_auto-weka:_2013} considers all 39 ML classification algorithms implemented in Weka to automatically and simultaneously choose a learning algorithm. Auto-WEKA uses sequential model-based optimization and a random forest regression model to approximate the dependence of a model's accuracy on the algorithm and hyper-parameter values. Using an approach similar to that in Auto-WEKA, Komer \textit{et al.} \cite{komer_hyperopt-sklearn:_2014} developed the software \textit{hyperopt-sklearn}, which automatically selects ML algorithms and the hyper-parameter values for Scikit-learn. 

In another work, Sparks \textit{et al.} proposed MLbase \cite{sparks_mli:_2013}, an architecture for automatically selecting ML algorithms, that supports distributed computing on a cluster of computers by combining better model search methods, bandit methods, batching techniques, and a cost-based cluster sizing estimator. 

Lokuciejewski \textit{et al.} \cite{lokuciejewski_automatic_2010} presented a generic framework for automatically selecting an appropriate ML algorithm for the compiler generation of optimization heuristics. Leite \textit{et al.} \cite{leite_selecting_2012} proposed a method called active testing for automatically selecting ML algorithms, that exploits metadata information concerning past evaluation results to recommend the best algorithm using a limited number of tests on the new dataset.

Van Rijn \textit{et al.} \cite{van_rijn_fast_2015} proposed a method for automatically selecting algorithms. They addressed the problem of algorithm selection under a budget, where multiple algorithms can be run on the full data set until the budget expires. Their method produces a ranking of classifiers and takes into account the run times of classifiers. 

\subsection{ML models for Resource-constrained Devices}

Researchers have worked on the inference problem on tiny resource-constrained IoT devices, which are not necessarily always-connected to the cloud. Kumar \textit{et al.} \cite{kumar_resource-efficient_2017}\cite{gupta_protonn:_2017} developed tree and k-nearest neighbor based algorithms, called Bonsai and ProtoNN respectively, for classification, regression, ranking, and other common IoT tasks. Their algorithm can be trained on the cloud and then be hosted onto resource-constrained IoT devices based on the Arduino Uno board. Bonsai and ProtoNN maintain prediction accuracy while minimizing model size and prediction costs. Motamedi \textit{et al.} \cite{motamedi_machine_2017} presented a framework for the synthesis of efficient Convolutional Neural Networks (CNN) inference software targeting mobile System on Chip (SoC) based platforms. They used parallelization approaches for deploying a CNN on SoC-based platforms. Meng \textit{et al.} \cite{meng_two-bit_2017} presented Two-Bit Networks (TBNs) approach for CNN model compression to reduce the memory usage and improve computational efficiency in terms of classification accuracy on resource-constrained devices. They utilized parameter quantization for computation workload reduction. Shoeb \textit{et al.} \cite{shoeb_application_2010} present an ML approach on a wearable device to identify epileptic seizures through analysis of the scalp electroencephalogram, a non-invasive measure of the brains electrical activity.

\begin{figure*}
\centering
\includegraphics[width=1\textwidth]{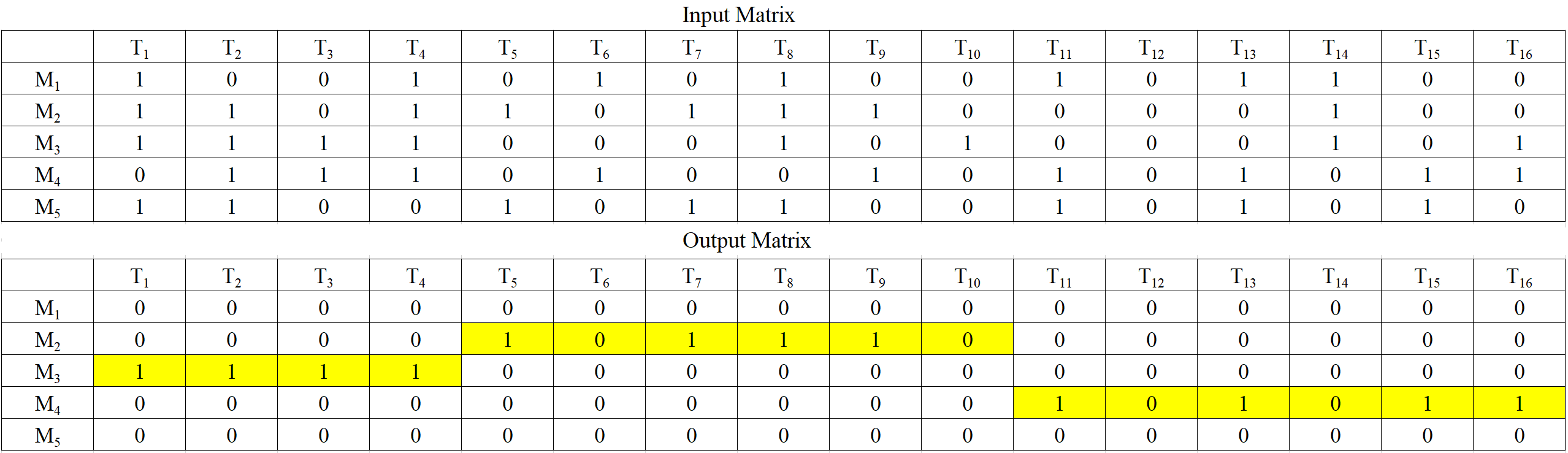}
\caption{An illustration of our proposed splice heuristic work. Here, the splice heuristic selects 3 longest consecutive sequences of 1s segments, then merges the adjacent unselected segments.}
\label{Fig_4.png}
\end{figure*}

\section{System Model and Problem Formulation}
\label{SystemModelSec}
In this paper, we assume an MLaaS provider that has $M$ ML models from which a subset needs to be selected and deployed on IoT devices for $T$ time slots. $P$ is a constant matrix of size $M$ $\times$ $T$, where element $p_{i,j}$ indicates the trust value obtained by model $i$ at time $j$. This matrix is created based on recent and past historical data that measure the degree of agreement of ML model $i$ with the other $M-1$ models in the ensemble of ML models. $B$ is the maximum number of allowed ML model reconfigurations during $T$ time slots. $A$ is a variable matrix of size $M$ $\times$ $T$, where element $a_{i,j} \in \{0,1\}$. $a_{i,j}=1$ indicates that the model $i$ at time $j$ is trustworthy; $a_{i,j}$ is equal to zero otherwise. In other words, A is a variable selection matrix where a value of 1 in row $r$ column $c$ indicates that model number $r$ is selected at time $c$; otherwise, if the value is 0 then model $r$ is not selected at time slot $c$. Thus, the objective of the formulation is to find the values of $a_{i,j}$ and $p_{i,j}$ such that the selected models maximize the overall trust values during the entire time period as shown in Equation (2).

As our proposed heuristic depends on the prediction output of ML model, it can be used with any supervised ML algorithms as classification and regression problems. In our experiments we use the proposed approach to select trusted LSTM models in regression problems. To compute the trust level of model $i$ at time $j$, we use Equation (\ref{trustLevel_equ}), which assigns a higher trust metric to models that agree more with the average of all models. This equation is inspired from the majority voting approach presented in the literature to quantify the level of trust \cite{li_distributed_2017} \cite{cho_trust_2016} \cite{raya_data-centric_2008} \cite{srinivasan_drbts:_2006} \cite{shahzad_comparative_2013}. Therefore, model $i$ at time $j$ is assigned trust level $p_{i,j}$ that represents the degree of  agreement  (i.e., reciprocal of the degree of deviation  $\frac{\sum_{k=1}^M d(O_{i,j}, O_{k,j})}{M}$) of  model $i$ with  other  models  in  the ensemble  of  ML  models. $d(O_{i,j}, O_{k,j})$ is a function that provides the distance between $O_{i,j}$ and $O_{k,j}$. The trust level metric ranges from $0$ to $p_{max}$ where a higher value indicates a higher level of trust.

\begin{IEEEeqnarray}{l}
\label{trustLevel_equ}
\qquad p_{i,j} = \min \bigg( p_{max} \ , \left[{\frac{\sum_{k=1}^M d(O_{i,j} ,O_{k,j})}{M}}\right]^{-1} \bigg),
\end{IEEEeqnarray}
where $p_{max}$ is the maximum attainable trust level in the given application domain and $p_{i,j}$ is the trust level of model $i$ at time $j$, $O_{i,j}$ is the output of model $i$ at time $j$, and $O_{i,j}$ $\in \mathbb{R}$.

%%%%%%%%%%%%%%%%%% Comment here

%%%%%%%%%%%%%%%%%%%%

\textbf{Problem Formulation: }\label{problem_formulation} The goal of this work is to maximize the trust level gained by selecting a subset of ML models from a superset of models to be hosted on resource-constrained devices for a period of time $R$, where $0 \leq R \leq T$. The number of reconfigurations is limited to $B$ and the maximum rate of reconfiguration is limited to $R$. We formulate the problem using ILP as follows:

\begin{IEEEeqnarray}{l}
\max\sum_{j=1}^T \sum_{i=1}^M a_{i,j} \cdot p_{i,j}
\label{main_obj}
\\
% Constraint 1
\mathrm{s.t.} \quad \sum_{i=1}^M a_{i,j}=1 \qquad \forall j \in 1 \ldots T,
\label{main_con_1}
\\
% Constraint 2
\qquad a_{i,j} \in \{0,1\} \qquad \forall i \in 1 \ldots M \nonumber \\
\qquad \qquad \qquad \qquad \quad \forall j \in 1 \ldots T,
\label{main_con_2} \\
% Constraint 3
\qquad \frac{1}{2} \cdot \sum_{i=1}^M \sum_{j=2}^{T} |a_{i,j}-a_{i,j-1}| \leq B,
\label{main_con_3} \\
% Constraint 4
\qquad \frac{1}{2} \cdot \sum_{i=1}^M \sum_{j=k}^{k+\ceil*{\frac{T}{B}}} |a_{i,j}-a_{i,j-1}| \leq R	\nonumber \\
\qquad \qquad \qquad \qquad \forall k \in 1 \ldots (T-\frac{T}{B}),
\label{main_con_4}
\end{IEEEeqnarray}

The first constraint in (\ref{main_con_1}) is to ensure that only one ML model is selected at each time slot, because there will be only one ML model hosted in a resource-constrained device at a time. The second constraint in (\ref{main_con_2}) indicates that this formulation is combinatorial, where the values can either be $0$ or $1$ with $1$ indicating the trustworthiness of the ML model and $0$ indicating that it is not. In order to comply with the maximum number of allowed reconfigurations ($B$), the third constraint in (\ref{main_con_3}) is used. The fourth constraint in (\ref{main_con_4}) restricts the solution to adhere to the models' maximum reconfiguration rate $R$ (i.e., the maximum number of reconfiguration per time unit). Table \ref{Table_1} summarizes the description of the formulation parameters. 

\begin{table}
\centering
\caption{Description of Formulation Parameters}
\label{Table_1}
 \begin{tabular}{|c | p{7cm}|} 
 \hline
 Parameter & Meaning \\ [0.5ex] 
 \hline\hline
 A & Variable matrix of size M$\times$T, $a_{i,j} \in \{0,1\}$  \\
 \hline
  B & Maximum number of allowed ML model configurations  \\
 \hline
  H & Constant value which represents the threshold of maximum trust level value selected from the fractional solution \\
 \hline
  $\epsilon$ & Constant small value that is subtracted from $H$ value   \\
 \hline
  M & Number of ML models   \\
 \hline
  $O_{i,j}$ & Output of the model $i$ at time $j$   \\
 \hline
  P & Constant matrix of size M$\times$T, which represents the trust value obtained by all models at all time slots  \\
 \hline
  $P_{i,j}$ & Trust value obtained by model $i$ at time $j$   \\
 \hline
  $P_{Max}$ & Maximum attainable trust level   \\
 \hline
  R & Maximum rate of reconfiguration   \\
 \hline
  T & The number of time slots   \\[1ex]
 \hline
\end{tabular}
\end{table}

\begin{algorithm}[t]
\caption{\textbf{Splice Heuristic} to find a lower-bound solution}
\label{alg_Splice_Heuristic}
\begin{algorithmic}[1]
\STATEx \textbf{Input}: Matrix $A$ of size $M \times T$ where element $a_{i,j}$ represents the trust level of model $i$ at time slot $j$;\newline
Maximum number of allowed reconfigurations $B$;\newline
Maximum reconfiguration rate $R$.
\STATEx \textbf{Output}: matrix $A$, with each column having only $1$ entry to indicate the selected ML model at the given time slot.
%\STATEx $\qquad \qquad \qquad \qquad \qquad$
\STATE Mark $A$ as one unselected segment
\STATE Set $i=0$
\WHILE {$i \leq B$ AND number of unselected segments $>$ 0}
\STATE Set $\textit{flag}$ = False
\STATE Identify unselected segment $j$ with at least $R$ columns that has the longest consecutive sequence of $1$s in row $k$.
\IF {Segment $j$ exists}
\STATE Set all entries of row $k$ to $1$, and set all entries of other rows of segment $j$ to $0$
\STATE Mark segment $j$ as selected
\IF {$w$ is a selected segment that is adjacent to $j$ and both have $1s$ in the same row}
\STATE Merge segments $w$ and $j$
\STATE Set $\textit{flag}$ = True
\ENDIF
\ENDIF
\IF {$\textit{flag}$ = False}
\STATE Set $i=i+1$
\ENDIF
\ENDWHILE
\STATE Merge adjacent unselected segments into one
\FOR {every unselected segment $j$}
\STATE Set $\textit{leftSum}$ = 0, $\textit{rightSum}$ = 0, $\textit{selectedRow}$ = 0
\IF {there is a selected segment $w$ with selected row $k$ left adjacent to segment $j$}
\STATE Set $\textit{leftSum}$ = sum of values of row $k$ in segment $j$
\STATE Set $\textit{selectedRow}$ = $k$
\ENDIF
\IF {there is a selected segment $w$ with selected row $z$ right adjacent to segment $j$}
\STATE Set $\textit{rightSum}$ = sum of values of row $z$ in segment $j$
\IF {$\textit{rightSum}$ $>$ $\textit{leftSum}$}
\STATE Set $\textit{selectedRow}$ = $z$
\ENDIF
\ENDIF
\STATE Set all entries of $\textit{selectedRow}$ of segment $j$ to $1$ and all entries of the other rows to $0$
\ENDFOR
\STATE Return $A$ as the best solution.
\end{algorithmic}
\end{algorithm}

\section{Proposed Solution}
\label{ProposedSolSec}
In this section, we discuss the proposed algorithms for the lower
bound and competitive solution along with the upper bound algorithm. We also illustrate the proof of NP-completeness of selecting a subset of ML models from a superset of ML models in order to maximize the trust level of ML models.

\subsection{Lower Bound}

To find a lower bound solution, we propose the \textbf{Splice Heuristic} shown in Algorithm (\ref{alg_Splice_Heuristic}). The heuristic accepts $A$, a matrix of size $M \times T$, where the element $a_{i,j}$ represents the trust level of model $i$ at time slot $j$. Initially, the heuristic considers $A$ as one unselected segment. Next, the heuristic iteratively uses three steps. In the first step, for each unselected segment that is at least $R$ in length, the heuristic finds the model (row) $k$ with the longest consecutive sequence of $1$s (i.e. the highest trust level). In the second step, the segment that has the highest trust level is marked as selected. Additionally, the row $k$ is selected by setting all the values in row $k$ to $1$ and in rows other than $k$ to $0$. The third step merges adjacent selected segments (from the previous rounds) into a single selected segment if they share the same selected model. 

These three steps are repeated until at most $B$ segments are selected or on unselected segments are left. Finally, the heuristic identifies unselected segments, if such segments exist. For each unselected segment, the heuristic finds the trust level using the highest trust level from a selected adjacent segment, if one exists. Finally, the heuristic compares the trust level resulting from the adjacent ML models (if they exist) and chooses the one with the highest trust level.

The example in Figure \ref{Fig_4.png} illustrates the details of our proposed Splice heuristic.  In this example, we assume that $R = 4$, $B = 2$. Consequently, the heuristic selects $B + 1 = 3$ segments that maximize the trust level. The first section covers time slots $T_1 - T_4$ with the selected ML model $M_3$. $M_2$ is selected in the second segment, which covers the time slots $T_7 - T_{10}$. Finally, the last segment has $M_4$ selected in the time slots $T_{13} - T_{16}$. After that the heuristic determines which ML model to use for the remaining unselected segments. For the time slots $T_5 - T_6$, $M_2$ is selected based on the selected adjacent segment to the right. In addition, for the time slots $T_{11} - T_{12}$, $M_4$ is selected based on the selected adjacent segment to the right.

\subsection{Upper Bound}

We relax the ILP formulation presented in Section \ref{problem_formulation} to a Linear Programming (LP) problem by replacing the constraint (\ref{main_con_2}) with 
\begin{IEEEeqnarray}{l}
\qquad a_{i,j} \in [0,1], \qquad \forall i \in 1 \ldots M \label{main_con_2_r}.
\end{IEEEeqnarray}

This relaxed formulation produces an upper bound solution for our problem.

\begin{algorithm}[t]
\caption{\textbf{Fixing Heuristic} to produce a competitive solution}
\label{alg_Fixing_Heuristic}
\begin{algorithmic}[1]
\STATEx \textbf{Input}: Matrix $A$ of size $M \times T$ where element $a_{i,j}$ represents the trust level of model $i$ at time slot $j$;\newline
Maximum number of allowed reconfigurations $B$;\newline
Maximum reconfiguration rate $R$;\newline
Maximum trust level selected from fractional solution $H$;\newline
Epsilon $\epsilon$, a small value subtracted from $H$.\newline 
\STATEx \textbf{Output}: matrix $A$ with each column having only one value as $1$ to indicate the selected ML model at the given time slot.
\STATEx $\qquad \qquad \qquad \qquad$ \textbf{PART I - Fixing}
\STATE Set $XSplice=A$
\STATE Run the Splice heuristic on matrix $XSplice$
\STATE Set $\textit{PreviousTrustLevel}$ = element-wise sum of $A \And XSplice$ where $\And$ is the bitwise AND operator
\STATE Set $\textit{XFraction} = A$ and apply linear programming to generate a fractional solution
\STATE Set $X = \textit{XFraction}$
\STATE Compute $\textit{CurrentTrustLevel}$ using PART II
\WHILE {$H > 0$ AND equation 1 through 3 are satisfied AND $\textit{CurrentTrustLevel}$ $>$ $\textit{PreviousTrsutLevel}$}
\STATE Set $\textit{PreviousTrustLevel}$ = $\textit{CurrentTrustLevel}$
\STATE Set $H = H - eps$
\STATE Set $X = \textit{XFraction}$
\STATE Compute $\textit{CurrentTrustLevel}$ using PART II
\ENDWHILE
\STATE Return $X$ as the best solution.
\STATEx
\STATEx $\qquad \qquad$ \textbf{PART II - Computing CurrentTrustLevel}
\STATE Set $\textit{rowNum}$ = -1
\FOR {$t$ from $t_0$ to $T$}
\IF {maximum value from column $t$ of matrix $X \geq$ H}
\STATE Set this maximum value to $1$ and set rest of values in column $t$ to $0$
\STATE Set $\textit{rowNum}$ = row number of the maximum value 
\ELSIF {$\textit{rowNum} = -1$}
\STATE Set the maximum value to $1$ and set rest of values in column $t$ to $0$
\ELSE
\STATE Set the value at $\textit{rowNum}$ to $1$ and set rest of values in column $t$ to $0$
\ENDIF
\ENDFOR
\STATE Set $\textit{CurrentTrustLevel}$ = element-wise sum of $A \And X$ where $\And$ is the bitwise AND operator
\end{algorithmic}
\end{algorithm}

\subsection{Competitive Solution}

To produce a competitive solution, we propose the \textbf{Fixing Heuristic} shown in Algorithm (\ref{alg_Fixing_Heuristic}). The algorithm accepts the matrix $A$, of dimensions $M \times T$ where the element $a_{i,j}$ represents the trust level of model $i$ at time slot $j$. The heuristic selects a maximum of $B+1$ ML models (which results in a maximum of $B$ model reconfigurations) to be used during $T$ in order to maximize the overall trust level. The proposed heuristic employs two constants: (1) a threshold $H$ that represents the maximum trust level selected from the fractional solution ($0 <$$H$$ < 1$); and (2) epsilon $\epsilon$ which is a small value that is subtracted from the value of $H$ during each iteration of the fixing process ($0<$$\epsilon$$<0.1$).

The proposed heuristic finds the lower bound solution first using the Splice heuristic \ref{alg_Splice_Heuristic} on matrix $A$. Next, the proposed fixing heuristic applies LP on matrix $A$ to find a fractional upper bound solution using $H$. Actually, the ML model with the highest trust level in each time slot of $A$ is compared with $H$. The highest trust level is rounded to $1$ if it is greater than or equal to $H$ while setting all other ML models to $0$ during that time slot. The same process is applied for trust levels less than $H$. If the highest trust level is less than $H$ in any time slot, the selected ML model in the previous time slot is selected for this time slot and is rounded to $1$ while other ML models are set to $0$. After converting the matrix into a binary one (i.e., $0$ or $1$ entries), the upper bound solution is computed by counting the number of entries in $A$ that are set to $1$. If the upper bound solution is found to be greater than the lower bound solution, the lower bound solution is set to the value of the upper bound solution. Also, $H$ is reduced by $\epsilon$ and the upper bound solution is recomputed in the hope of finding a better solution. This process is repeated as long as the upper bound solution is improved.

Because our proposed algorithm depends on the solution produced by Linear Programming (LP), which can be solved using the Simplex algorithm, then the complexity of our proposed algorithm is similar to the complexity of the Simplex algorithm which has polynomial-time complexity under various probability distributions.

\subsection{Proof of NP-completeness}

In this section, we show that the problem discussed in this paper can be reduced from the decision version of the set cover problem, which is known to be NP-complete.

We define the universe $\mathcal{U}$ as a set of tuples $(i, j), i,j \in T$ and $i \leq j$. Each tuple $(i, j)$ represents a time interval that starts at time $i$ and ends at time $j$ during which the system uses the same model without any reconfigurations. We also define $\mathcal{S}$ as a family of subsets of $\mathcal{U}$. The union of $\mathcal{S}$ results in a period that covers $\mathcal{U}$. In other words, the union of $\mathcal{S}$ results in a period that starts at time $0$ and ends at time $T$. Now the cardinality of $S$ is represented as follows:

\begin{IEEEeqnarray}{l}
0 \leq ||S|| \leq \left[\sum_{i=1}^T  \binom{T}{i}\right] * M.
\label{main_obj_r}
\end{IEEEeqnarray}

%%%%%%%%%%%%%%%%  Commented %%%%%%%%%%%%%%%%%%%
\begin{comment}
\begin{IEEEeqnarray}{l}
0 \leq ||S|| \leq \left[\sum_{i=1}^T \left \binom{T}{i} \right)\right] \cdot M.
\label{main_obj_r}
\end{IEEEeqnarray}

\begin{IEEEeqnarray}{l}
0 \leq ||S|| \leq \left[\sum_{i=1}^T \left \binom{T}{i} \right)\right] \cdot M
\label{main_obj_r}
\end{IEEEeqnarray}
\end{comment}

%%%%%%%%%%%%%%%%%%%%%%%%%%%%%%%%%%%

%\Mycomb[T]{i}

If $k$ represents the maximum number of model reconfigurations, the objective of our problem is to find $k$ subsets from $\mathcal{S}$ while maximizing the total trust level. This problem is similar to the decision version of the set cover problem. The universe $\mathcal{U}$ and the set $\mathcal{S}$ of our problem are the same as the universe $\mathcal{U}$ and set $\mathcal{S}$ in the set cover problem. However, in our problem, every element is a tuple. The maximum number of model reconfigurations $k$ is the same as the integer number $k$ in the set cover problem. Consequently, the problem introduced in this paper is NP-complete.

\subsection{Worst-Case Analysis (Competitive Ratio Analysis)}

The performance of our proposed fixing heuristic is at least as good as that of the splice heuristic. Consequently, the worst case scenario is encountered when the proposed heuristic performs as the splice heuristic. When the maximum number of allowed reconfigurations is set to $B$, the splice heuristic finds ($B+1$) segments, each of length $R$, that provide the maximal trust level.

\begin{prop}
For any configuration $X$, the splice heuristic's worst-case performance has a competitive ratio of $O(1)$ when $R$ is proportional to $T$ and $B$ is constant.
\end{prop}

\begin{proof}
Let ALG be the splice heuristic and OPT be the optimal algorithm. The first part of the splice heuristic (Algorithm \ref{alg_Splice_Heuristic}), specifically steps 1 through 17, finds the segments with the longest consecutive sequence of $1$s. Actually, both ALG and OPT select those segments since they have the largest sum of values (i.e., maximum trust levels). Specifically, those segments have a total length of $R(B+1)$. However, the two approaches differ in the rest of the solution, which is the unselected segments in ALG. Now, at the end of the first part and in the worst-case scenario, Algorithm (\ref{alg_Splice_Heuristic}) may already have performed $B$ reconfigurations and cannot use more reconfigurations. In other words, for every unselected segment, Algorithm (\ref{alg_Splice_Heuristic}) can only use either of the selected models in the adjacent selected segments but never a different model. Consequently, in the worst-case scenario, the two models in the selected segments adjacent to the unselected segment are different. Thus, the second part of Algorithm (\ref{alg_Splice_Heuristic}), specifically steps 17 through 32, will pick the model that has the largest sum in the unselected segment. In the worst-case scenario, both the left and right adjacent selected segments may have the same value when both used in the unselected segment and therefore Algorithm (\ref{alg_Splice_Heuristic})'s maximum loss is half the segment length. However, the loss can never be less than half the segment length. Mathematically, in the second part of the solution, OPT achieves a maximum of $T-R(B+1)$ while ALG achieves a minimum of $\frac{1}{2}[T-R(B+1)]$. Consequently, the following proof is concluded as follows.

Competitive Ratio = $\frac{ALG(X)}{OPT(X)}$
\begin{IEEEeqnarray*}{l}
=\frac{R(B+1)+\frac{1}{2}[T-R(B+1)]}{R(B+1)+[T-R(B+1)]}\\
= \frac{\frac{1}{2}T+\frac{1}{2}R(B+1)}{T}\\
= \frac{1}{2}\frac{T+R(B+1)}{T}\\
= \frac{1}{2}\left[\frac{T}{T}+\frac{RB}{T}+\frac{R}{T}\right]\\
= O\bigg(\frac{R(B+1)}{T}\bigg)
\end{IEEEeqnarray*}
This competitive ratio is $O(1)$ when $R$ is proportional to $T$ and $B$ is constant.
\end{proof}

\begin{figure*} 
 \centering
\includegraphics[width=1\textwidth]{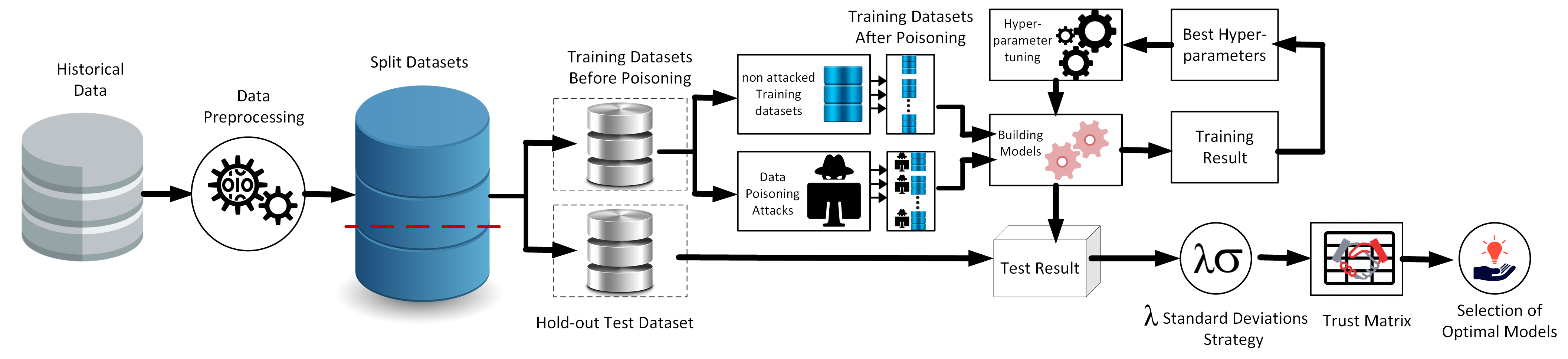}
\caption{The data processing pipeline utilized in our experimental studies starting from the data collection phase and ending with the selection of trustworthy ML models.}
\label{Fig_5.png}
\end{figure*}

\section{Performance Evaluation}
\label{PerformanceEvaSec}
%\subsection{Experimental Setup}

In order to evaluate the performance of the proposed heuristic, we designed and implemented the data processing shown in Figure \ref{Fig_5.png}. In our experiments, we focused on two case studies that serve as proxies for smart city and IIoT services. In our experiments, we trained multiple ML models using sampled experimental datasets to simulate multiple service providers sending ML models to resource-constrained devices.
 
%\subsubsection{Datasets}

\subsection{Experimental Setup}

The first case study is a proxy for smart city services in which the City Pulse EU FP7 project \cite{noauthor_citypulse_nodate} dataset is used for traffic prediction. This dataset conveys the vehicular traffic volume collected from the city of Aarhus, Denmark, observed between two points for a set duration of time over a period of $6$ months.
 
The second case study is a proxy for IIoT services in which the Turbofan engine degradation simulation dataset, provided by the Prognostics CoE at NASA Ames \cite{saxena_turbofan_2008}, is used for predicting the remaining useful life of engines. Engine degradation simulation was carried out using a C-MAPSS tool. The goal is to predict the remaining useful life, or the remaining number of cycles before the turbofan engine reaches a level that no longer performs up to requirements. The requirement is based on data collected from sensors located on the turbofan and also on the number of cycles completed. The prediction helps to plan maintenance in advance. The training data consists of multiple multivariate time series with ``cycle'' as the time unit, together with $21$ sensor readings for each cycle. Each time series can be assumed as being generated from a different engine of the same type. The testing data has the same data schema as the training data. The only difference is that the data does not indicate when the failure occurs. Finally, the ground truth data provides the number of remaining work cycles for the engines in the test data. Table \ref{Table_2} shows the description of datasets for both use case studies. 

\begin{table}[h]
\centering
\scriptsize
\caption{Description of datasets for case studies }
\label{Table_2}
 \begin{tabular}{|p{1cm} | p{2cm} | p{2cm} | p{2cm} |} 
 \hline
Dataset & Number of records in each training sample & Number of records in testing set & Number of features\\ [0.5ex] 
 \hline\hline
Traffic volume & One week of hourly counting the number of vehicles & Seven weeks of counting the number of vehicles & 12 lags of number of vehicles \\
\hline
  Turbofan engine degradation & 3716 records sampled from 250 engines & 1800 records sampled from 250 engines & 27 features include engine id, cycle number, 3 settings,  21 sensors readings, and remaining useful life (RUL) \\ [1ex]
\hline
\end{tabular}
\end{table}

Each dataset is divided into training and testing subsets. Each training dataset is sampled into $27$ different datasets that we used to train $27$ deep LSTM models ($17$ of those models are benign and 10 are malicious, 20\% of the training data of the malicious models are poisoned with causative attacks). Even though it is possible to sample our experimental datasets differently to produce a higher/lower number of machine learning models, our choice of $27$ was based on exploratory experiments designed to explore the maximum number of ML models that can be produced from our experimental datasets without affecting the accuracy of the generated models. Specifically, we use the swap $x$ and $100-x$ percentiles attack model as a causative attack to intentionally poison the learners' classifications by altering the labels of the training dataset. In the swap $x$ and $100-x$ percentiles attack, the $x$ percentile value is exchanged with the $100-x$ percentile value. As an example of the swap $x$ and $100-x$ percentiles attack, consider the numeric dataset in Figure \ref{PercentileExample}. To find the $i^{th}$ percentile, we need to sort the values in the unsorted list in ascending order. Next, we multiply $i\%$ by the total number of items in the list (i.e., $10$ items). Now, for example, let us find $20^{th}$ and $80^{th}$ percentiles in the list. $20^{th}$ percentile = 0.2 $\times$ 10 = 2 (item index), which is value $174$ in the list. $80^{th}$ percentile = 0.8 $\times$ 10 = 8 (item index), which is value $188$ in the list. Now, to swap the $x$ and $100-x$ percentiles in this dataset, every $174$ will be replaced with $188$, and every $188$ will be replaced with $174$ in the region in which we want to introduce the swap $x$ and $100-x$ percentiles attack. 
 
 \begin{figure} [h]
\centering
\includegraphics[width=0.4\textwidth]{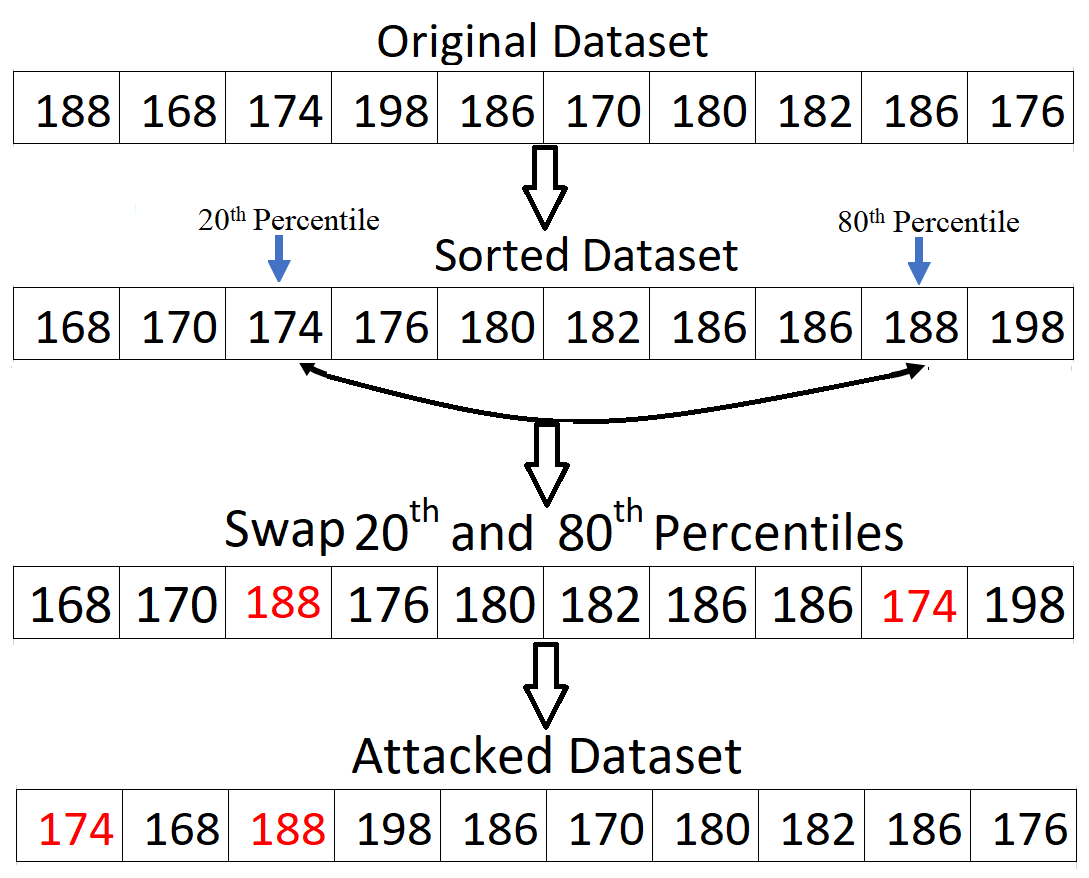}
\caption{Example of swap $x$ and $100-x$ percentiles attack model.}
\label{PercentileExample}
\end{figure}

Since our goal in this paper is to assess the trust level of LSTM models, we used grid search to tune the number of hidden layers, the number of neurons in a layer, the batch size, and the activation function parameters that play a major role in the building of LSTM models \cite{qolomany_parameters_2017} \cite{qolomany_role_2017} \cite{qolomany_leveraging_2019}. ML models are trained using different configurations. Each configuration includes different values for the number of hidden layers, the number of neurons in each layer, and activation functions. Finally, we select the configuration that gives the best accuracy. Table \ref{Table_3} shows the ranges of the configuration parameters used in our experiments to generate the LSTM models. 

\begin{table} [h]
\centering 
\caption{Configuration parameter ranges}
\label{Table_3}
 \begin{tabular}{|c | c|} 
 \hline
 Parameter & Value \\ [0.5ex] 
 \hline\hline
 Number of hidden layers & [1--6]  \\
 \hline
  Number of Neurons & [4--1024]  \\
 \hline
  Activation function & Rectified Linear Unit (ReLU) \\
 \hline
  Batch size & [72--200]\\
 \hline
 Epochs  & [10-50]   \\ %[1ex]
 \hline
\end{tabular}
\end{table}

After building the models, we evaluated our model selection approach on two experimental datasets. Every row in the traffic dataset represents the number of vehicles during a specific hour. On the other hand, a row in the Turbofan engine dataset represents the remaining useful life during a specific cycle. Figures \ref{Fig1_SC.png} to \ref{Fig4_SC.png} for the smart city traffic flow prediction use-case and Figures \ref{Fig1_IIoT.png} to \ref{Fig4_IIoT.png} for the IIoT predictive maintenance use-case show that the trust level varies between the two datsets. This is because the number of the observations in the test set is different for the two experimental datasets. For the traffic dataset, the number of observations is $\sim$900 while for the  turbofan dataset the number of observations is $\sim$1800. Next, we utilize $\lambda$ standard deviations strategy, which is inspired by the Six Sigma strategy  \cite{pyzdek_six_2003} to exclude the malicious models by identifying and removing the causes of defects and minimizing variability using statistical methods (namely, the mean and the standard deviation as shown in Equations (\ref{eq_outUpper}) and (\ref{eq_outLower})), which leads to better trust prediction models. 

\begin{IEEEeqnarray}{l}
\label{eq_outUpper}
OutUpper = \mu +  \lambda \times \sigma
\end{IEEEeqnarray}
\begin{IEEEeqnarray}{l}
\label{eq_outLower}
OutLower = \mu - \lambda \times \sigma
\end{IEEEeqnarray}

Every time step, we compute $\mu$, which is the mean of the outputs of all models. Also, we compute $\sigma$, the standard deviation of the outputs of all models. $\lambda$ defines the model exclusion strategy (i.e., any model that has an output that is $>$ $\mu + \lambda \times \sigma$ or is $<  \mu - \lambda \times \sigma$ is excluded). The $\lambda$ standard deviations strategy produces a trust matrix of size $M \times T$, with $1$ indicating a trusted model and $0$ indicating a malicious model. The resulting matrix is then used as the input (i.e., matrix $A$) for the proposed fixing heuristic.

\subsection{Experimental Results}

In this section, we discuss the results of using the proposed fixing heuristic along with the lower bound and upper bound heuristics on the two datasets introduced in the previous section.

\begin{figure*}[!h]
\centering
\minipage[b]{0.47\linewidth}
\includegraphics[width=\linewidth, height=5.6cm]{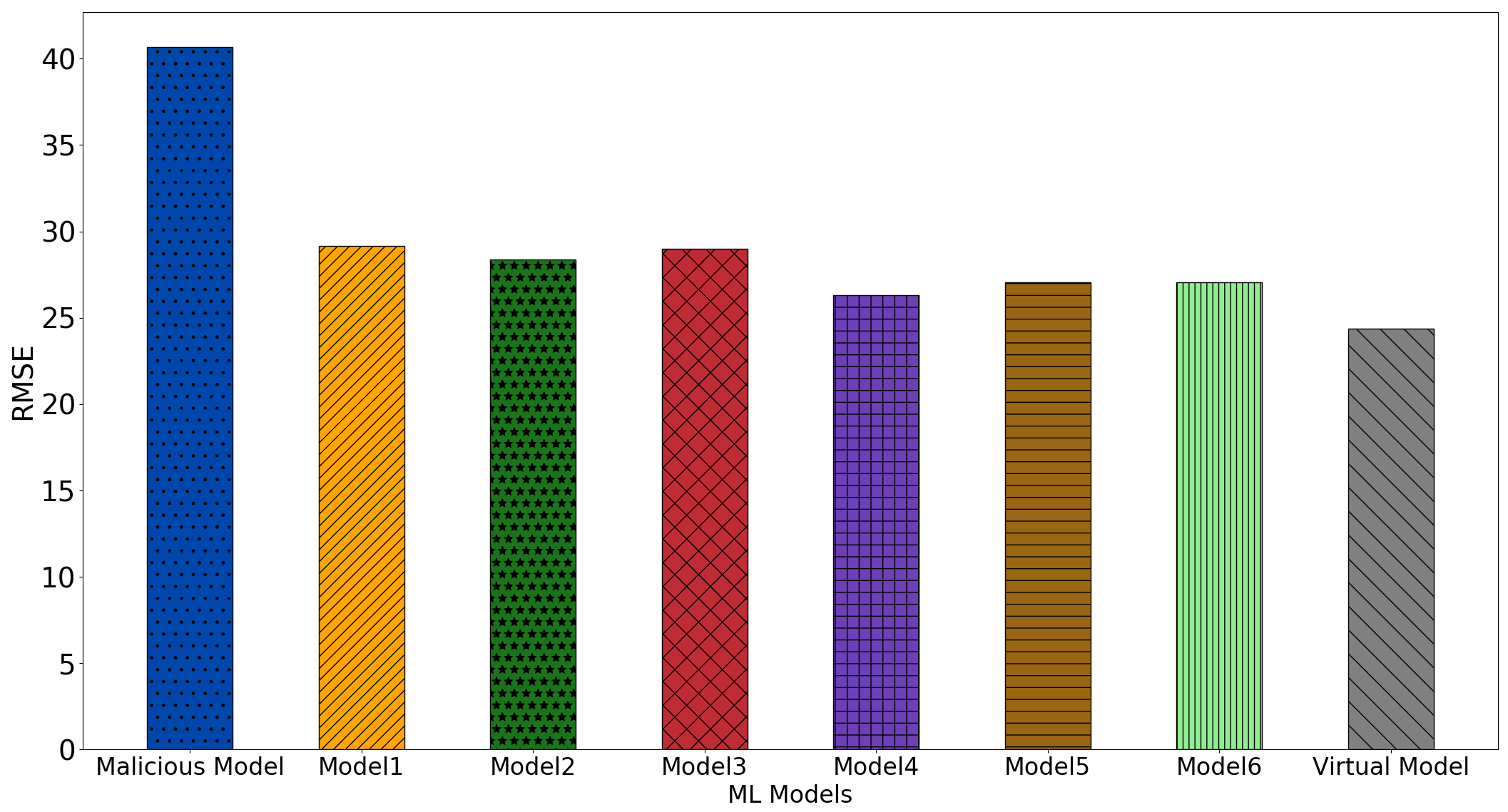}
\caption{Smart city traffic flow prediction use-case: RMSE of the models using the fixing heuristic vs. individual models.}
\label{Fig0_SCpng}
\endminipage\hfill
%\begin{subfigure}[b]{0.49\textwidth}
\minipage[b]{0.49\linewidth}
\includegraphics[width=\linewidth, height=6.5cm]{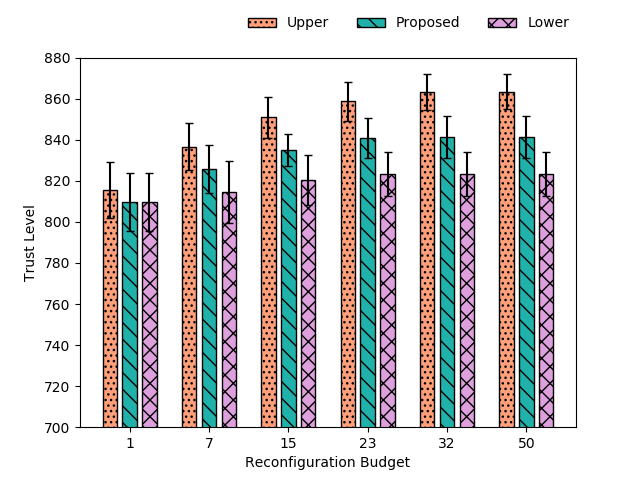}
\caption{Smart city traffic flow prediction use-case: Trust level of upper bound, lower bound, and proposed heuristics.}
\label{Fig1_SC.png}
\endminipage
%\end{subfigure}
%\begin{subfigure}[b]{0.49\textwidth}
%\hspace*{\floatsep}% https://tex.stackexchange.com/q/26521/5764
%\end{subfigure}
\end{figure*}

%%%%%%%%%%%%%%%%%%%%%%%%%%%%%%%%%

\begin{figure*}[!h]
\centering
\minipage[b]{0.47\textwidth}
\includegraphics[width=\textwidth]{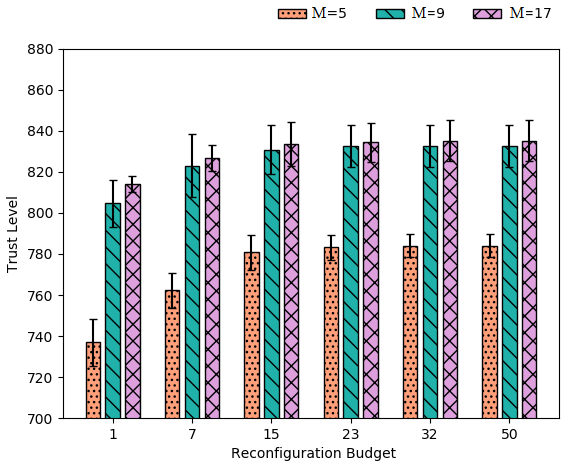}
\caption{Smart city traffic flow prediction use-case: The effect of the number of selected models on the trust level.}
\label{Fig2_SC.png}
\endminipage\hfill
%\begin{subfigure}[b]{0.49\textwidth}
\minipage[b]{0.47\textwidth}
\includegraphics[width=\textwidth]{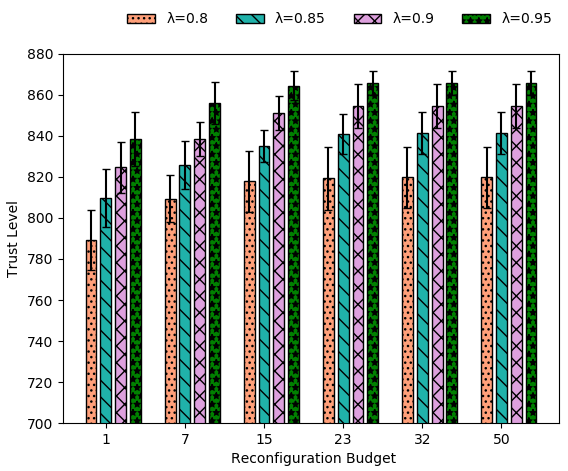}
\caption{Smart city traffic flow prediction use-case: The effect of $\lambda$ on the trust level.}
\label{Fig3_SC.png}
\endminipage
%\end{subfigure}
%\begin{subfigure}[b]{0.49\textwidth}
%\hspace*{\floatsep}% https://tex.stackexchange.com/q/26521/5764
%\end{subfigure}
\end{figure*}

%%%%%%%%%%%%%%%%%%%%%%%%%%%%%%%%%%%%

\begin{figure*}[!h]
\centering
\minipage[b]{0.46\textwidth}
\includegraphics[width=\textwidth]{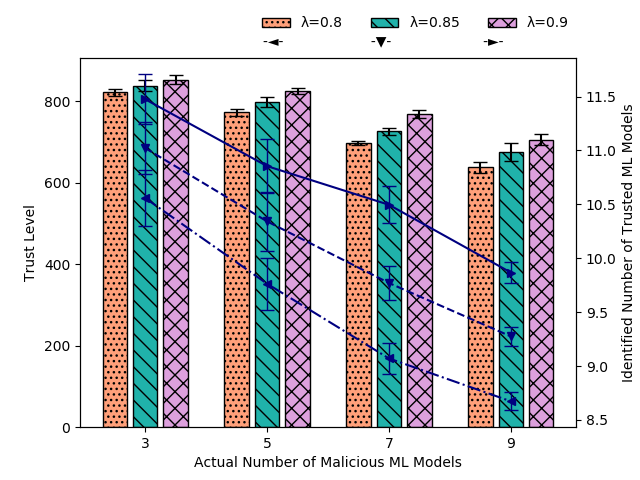}
\caption{Smart city traffic flow prediction use-case: The effect of malicious models on the trust level.}
\label{Fig4_SC.png}
\endminipage\hfill
%\begin{subfigure}[b]{0.49\textwidth}
\minipage[b]{0.45\textwidth}
\includegraphics[width=\textwidth, height=5.7cm]{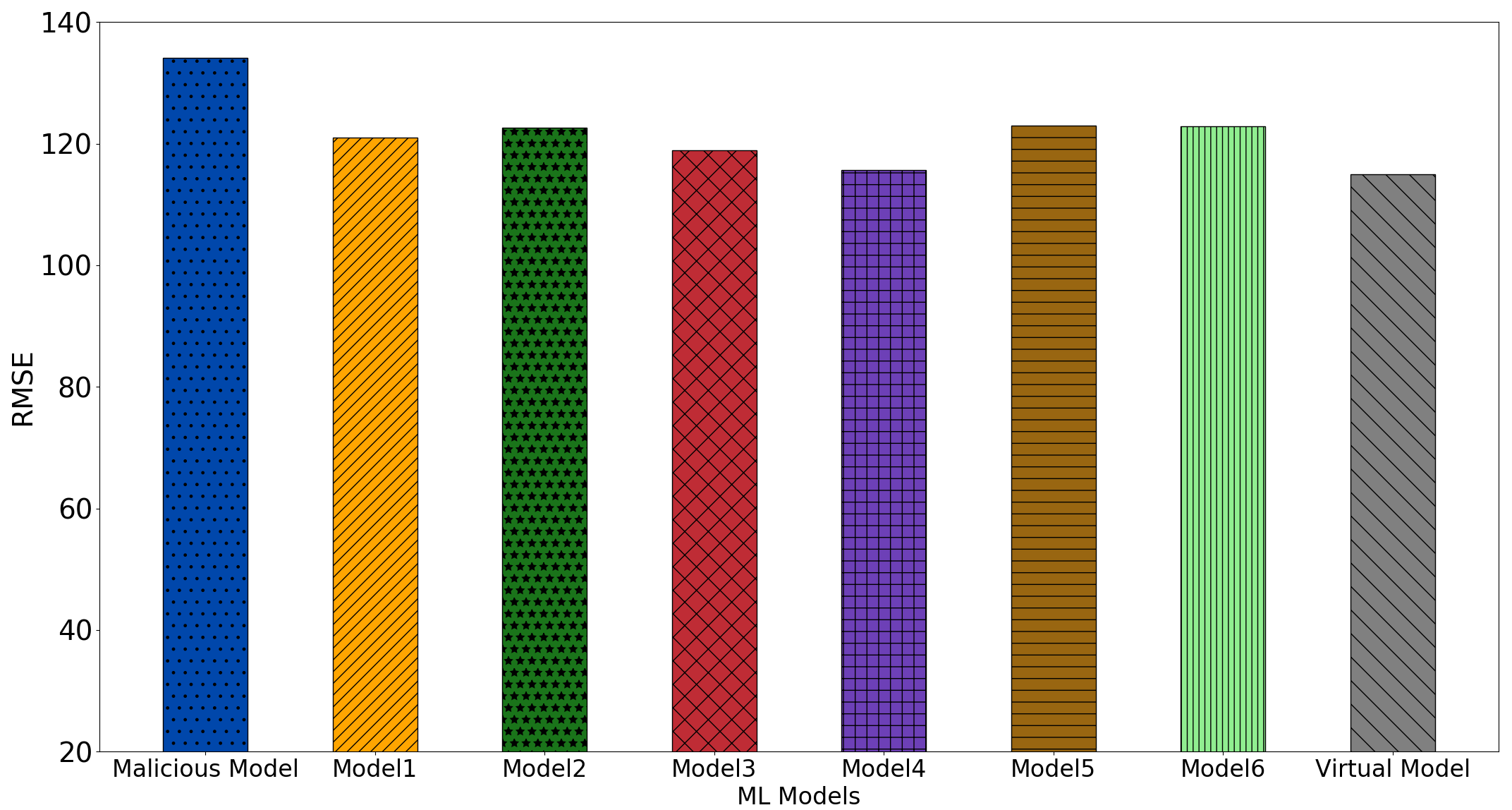}
\caption{IIoT predictive maintenance use-case: RMSE using the fixing heuristic vs. individual models.}
\label{Fig0_IIoT.png}
\endminipage
%\end{subfigure}
%\begin{subfigure}[b]{0.49\textwidth}
%\hspace*{\floatsep}% https://tex.stackexchange.com/q/26521/5764
%\end{subfigure}
\end{figure*}

\begin{figure*}[!h]
\centering
%\begin{subfigure}[b]{0.49\textwidth}
\minipage[b]{0.49\textwidth}
\includegraphics[width=\textwidth]{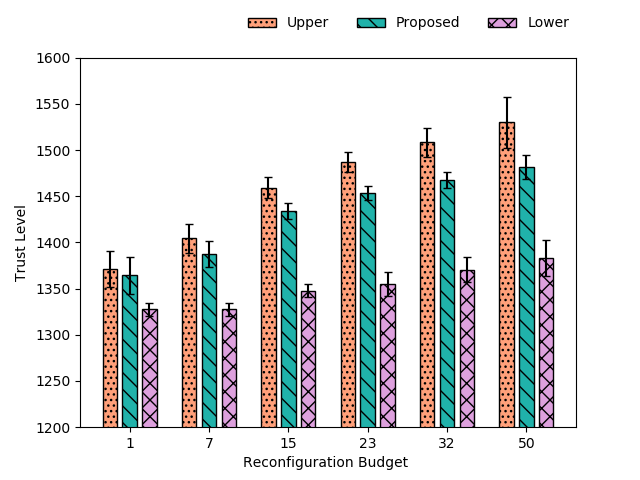} 
\caption{IIoT predictive maintenance use-case: Trust level of upper bound, lower bound, and proposed heuristics.}
\label{Fig1_IIoT.png}
\endminipage\hfill
\minipage[b]{0.45\textwidth}
\includegraphics[width=\textwidth]{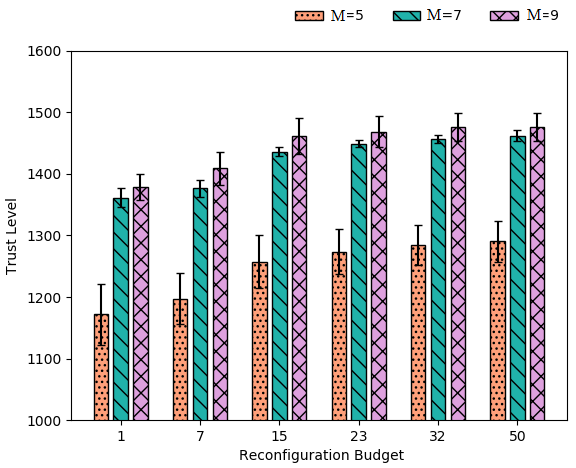}
\caption{IIoT predictive maintenance use-case: The effect of the number of selected models on the trust level.}
\label{Fig2_IIoT.png}
\endminipage
%\end{subfigure}
%\begin{subfigure}[b]{0.49\textwidth}
%\hspace*{\floatsep}% https://tex.stackexchange.com/q/26521/5764
%\end{subfigure}
\end{figure*}

\begin{figure*}[!h]
\centering
%\begin{subfigure}[b]{0.49\textwidth}
\minipage[b]{0.45\textwidth}
\includegraphics[width=\textwidth]{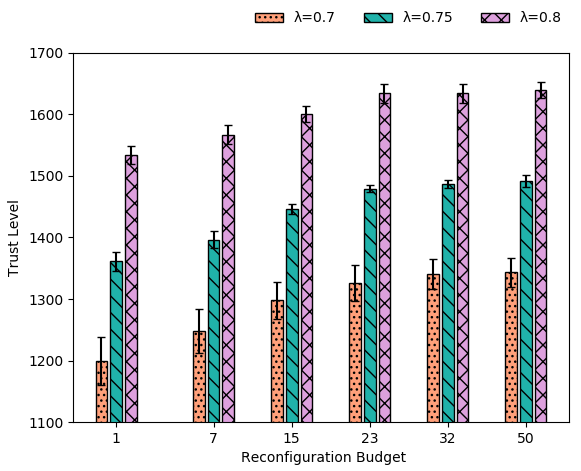}
\caption{IIoT predictive maintenance use-case: The effect of $\lambda$ on the trust level.}
\label{Fig3_IIoT.png}
\endminipage\hfill
%\end{subfigure}
%\begin{subfigure}[b]{0.49\textwidth}
%\hspace*{\floatsep}% https://tex.stackexchange.com/q/26521/5764
\minipage[b]{0.49\textwidth}
\includegraphics[width=\textwidth]{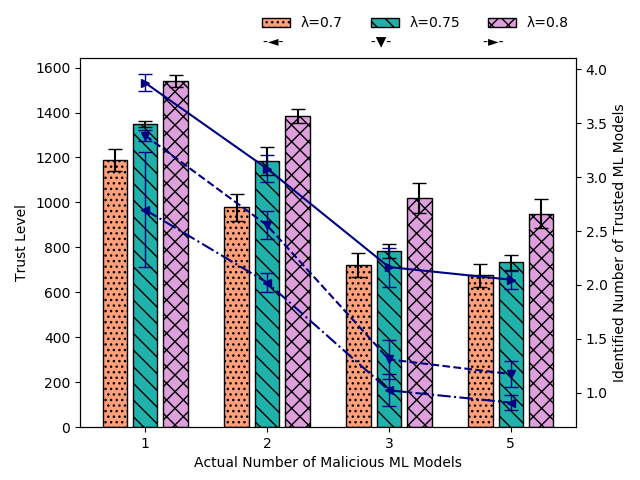}
\caption{IIoT predictive maintenance use-case: The effect of malicious models on the trust level.}
\label{Fig4_IIoT.png}
\endminipage\hfill
%\end{subfigure}
\end{figure*}
\subsubsection{Traffic Flow Volume Prediction}

In our first experiment, we studied the Root Mean Square Error (RMSE) of the models selected using our proposed fixing heuristic vis-\`a-vis the individual models. We set the reconfiguration budget $B$ to $7$ as shown in Figure \ref{Fig0_SCpng}. As the figure shows, the proposed fixing heuristic results in 11\%--66.95\% less RMSE when compared to the individual models. 

In our second experiment, the trust levels resulting from the three heuristics are compared under different reconfiguration budgets as illustrated in Figure \ref{Fig1_SC.png}. The figure shows the confidence interval for $5$ replications. In each replication, the malicious model is applied on a different model (e.g., $MM_1$, $MM_2$, $\dots$ , or $MM_n$). In this experiment, we set $\lambda$ to $0.85$, $M$ to $7$, and the number of malicious models $C$ to $1$. The number of non-malicious models is $M-C$. 

In our third experiment, the trust level of the selected models is studied as the number of models $M$ is varied ($5$, $9$, and $17$) as illustrated in Figure \ref{Fig2_SC.png}. In this experiment, we set $\lambda$ to $0.85$ and $C$ to $1$. In addition, the figure indicates that as $M$ is increased, the trust level of the selected models is increased too.

Figure \ref{Fig3_SC.png} shows the results of our fourth experiment. In this experiment, the effect of using different values of $\lambda$ ($0.8$, $0.85$, $0.9$, $0.95$) on the trust level of the selected models is analyzed. In this experiment, we set $C$ to $1$ and $M$ to $7$. The figure shows this effect for different reconfiguration budgets $B$. The figure indicates that as $\lambda$ is increased, the trust level of the selected models is increased too.  

Figure \ref{Fig4_SC.png} shows the effect of the number of the malicious LSTM models $C$ ($3$, $5$, and $7$) on the trust level of the selected models for different values of $\lambda$ ($0.8$, $0.85$, $0.9$). The figure also shows  the actual number of malicious LSTM models versus the identified number of malicious LSTM models. In this experiment, we set $M$ to $17$ and $B$ to $7$. 

\subsubsection{Predictive Maintenance in IIoT}

In our first experiment, we studied the Root Mean Square Error (RMSE) of the models selected using our proposed fixing heuristic vis-\`a-vis the individual models. We set the reconfiguration budget $B$ to $7$ as shown in Figure \ref{Fig0_IIoT.png}. As the figure shows, the proposed fixing heuristic results in $0.5\%$--$15\%$ less RMSE when compared to the individual models. 

In our second experiment, the trust levels resulting from the three heuristics are compared under different reconfiguration budgets as illustrated in Figure \ref{Fig1_IIoT.png}. The figure shows the confidence interval for $5$ different replications. In each replication, the malicious model is applied to a different model (e.g. $MM_1$, $MM_2$, $\dots$ , or $MM_n$). In this experiment, we set $\lambda$ to $0.75$, number of malicious models $C$ to $1$, and $M$ to $7$. The number of non-malicious models is $M - C$.

In the third experiment, the trust level of the selected models is studied as the number of models $M$ is varied ($5$, $7$, and $9$) as illustrated in Figure \ref{Fig2_IIoT.png}. In this experiment, we set $\lambda$ to $0.75$ and $C$ to $1$. In addition, the figure indicates that as $M$ is increased, the trust level of the selected models is increased too. 

%%%%%%%%%%%%%%%%%%%%%%%%

Figure \ref{Fig3_IIoT.png} shows the results of our fourth experiment. In this experiment, the effect of using different values of $\lambda$ ($0.7$, $0.75$, $0.8$) on the trust level of the selected models is analyzed. In this experiment, we set $C$ to $1$ and $M$ to $7$. The figure shows this effect given different reconfiguration budgets $B$. The figure indicates that as $\lambda$ is increased, the trust level of the selected models is increased too.  

Figure \ref{Fig4_IIoT.png} shows the effect of the number of the malicious LSTM models $C$ ($1$, $2$, and $3$) on the trust level of the selected models for different values of $\lambda$ ($0.7$, $0.75$, $0.8$). The figure also shows the actual number of malicious LSTM models versus the identified number of malicious LSTM models. In this experiment, we set $M$ to $7$ and $B$ to $7$.

\subsection{Discussion and Lessons Learned}

We can conclude the following based on the results presented in the previous section:

\begin{enumerate}
\item It is important to evaluate ML models used in critical and sensitive decisions in terms of trustworthiness and reliability. Additionally, other traditional criteria of ML model evaluation must be considered (e.g., accuracy, run time, etc.).

\item Our proposed fixing heuristic strives to maximize the trust level while not affecting the accuracy of the selected models, as Figures \ref{Fig0_SCpng} and \ref{Fig0_IIoT.png} indicate.

\item Figures \ref{Fig1_SC.png} and \ref{Fig1_IIoT.png} show that
our proposed fixing heuristic is able to obtain a trust level that is $0.7\%$--$2.53\%$ lower than that obtained by the upper bound solution in smart city case study, and $0.49\%$--$3.17\%$ lower than that obtained by the upper bound solution in IIoT case study. Figures \ref{Fig1_SC.png} and \ref{Fig1_IIoT.png} also indicate that by increasing the reconfiguration budget, the trust level is increased. However, there is a limit beyond which increasing the number of reconfigurations does not increase the trust level.

\item Figures \ref{Fig2_SC.png} and \ref{Fig2_IIoT.png} indicate that increasing the number of selected models lead to an increase in the trust level of the overall system. This fact is similar to the concept of evaluating the seller feedback on online shopping sites, restaurants, or hotels reviews. As the volume of feedback increases, the level of reliability of such reviews increases as well.  

\item Figures \ref{Fig3_SC.png} and \ref{Fig3_IIoT.png} indicate that increasing $\lambda$, the number of the excluded models is decreased. However, increasing $\lambda$ beyond a specific threshold may lead to the use of malicious models. On the other hand, using a small value for $\lambda$ leads to excluding more models, which might not be malicious.

\item  Figures \ref{Fig4_SC.png} and \ref{Fig4_IIoT.png} indicate that increasing the number of malicious models leads to a decrease in the trust level of the overall system. This is due to the fact that the proposed heuristic excludes malicious models and it might reach a fail-safe execution state in which it informs the resource-constrained devices that there are no trusted ML models to be hosted on them.     \end{enumerate}

\section{Conclusions and Future Works}
\label{ConclusionSec}
In this paper, we consider the paradigm in which resource-constrained IoT devices execute ML algorithms locally, without necessarily being connected to the cloud all the time. This paradigm is desirable in systems that have strict latency, connectivity, energy, privacy, and security requirements. There is a strong need in such environments to evaluate the level of trustworthiness of ML models built by different service providers, we formulate the problem of finding a subset of ML models that maximizes the trustworthiness while adhering to a given reconfiguration budget and rate constraints. We prove that this problem is NP-complete and propose a fixing heuristic that finds a near-optimal solution in polynomial time.

To measure the performance of the proposed fixing heuristic compared to integer linear programming (ILP), we applied our proposed fixing heuristic to two different case studies: (1) the traffic flow volume dataset to predict the number of vehicles (as a proxy case study for smart cities services); and (2) the turbofan engine degradation simulation dataset to predict the remaining useful life for the engine (as a proxy for  IIoT services). Our proposed fixing heuristic returns impressive performance achieving a high trust level that is less than the optimal ILP solution by only $0.7\%$--$2.53\%$ in the smart city service case study and $0.49\%$--$3.17\%$ less in the IIoT service case study.

There are a number of avenues of future work that can be pursued. Although we only use LSTM for developing the models in this paper, other types of models (e.g., CNN, deep neural networks, and SVM) can also be explored. It would be interesting to perform a comparative study of these models and also consider their robustness to adversarial attacks compared to our proposed fixing heuristic. Additionally, potential applications of our proposed heuristic can be explored in the speech, video, and medical domains, and in recommendation systems. 

\ifCLASSOPTIONcaptionsoff
  \newpage
\fi


%\vfill


%\bibliographystyle{IEEEtran}
\bibliographystyle{IEEEtran}
\bibliography{refrecovery.bib}
%\bibliography{main.bbl}
%\bibliography{TrustedML.bib}
%\bibliography{biblo}

\begin{IEEEbiography}[{\includegraphics[width=1in, height=1.25in,clip]{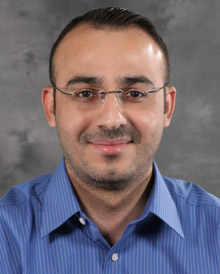}}]%
{Basheer Qolomany (S'17)}
received the Ph.D. and second master’s en-route to Ph.D. degrees in Computer Science from Western Michigan University (WMU), Kalamazoo, MI, USA, in 2018. He also received his B.Sc. and M.Sc. degrees in computer science from University of Mosul, Mosul city, Iraq, in 2008 and 2011, respectively. He is currently an Assistant Professor at Department of Cyber Systems, University of Nebraska at Kearney (UNK), Kearney, NE, USA. Previously, he served as a Visiting Assistant Professor at Department of Computer Science, Kennesaw State University (KSU), Marietta, GA, USA, in 2018-2019; a Graduate Doctoral Assistant at Department of Computer Science, WMU, in 2016-2018; he also served as a lecturer at Department of Computer Science, University of Duhok, Kurdistan region of Iraq, in 2011-2013. His research interests include machine learning, deep learning, Internet of Things, smart services, cloud computing, and big data analytics. Dr. Qolomany has served as a reviewer of multiple journals, including IEEE Internet of Things journal, Energies — Open Access Journal, and Elsevier - Computers and Electrical Engineering journal. He also served as a Technical Program Committee (TPC) member and a reviewer of some international conferences including IEEE Globecom, IEEE IWCMC, and IEEE VTC.

\end{IEEEbiography}

\begin{IEEEbiography}[{\includegraphics[width=1.1in,height=1.22in,clip]{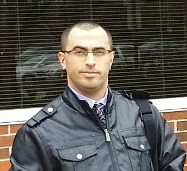}}]%
{Ihab Mohammed (S'14)}
    is a Ph.D. student at the NEST Research Lab in the Computer Science department of Western Michigan University, Kalamazoo, MI, USA. He received his B.S. and M.S. degrees in computer science from Al-Nahrain University in Iraq in 2002 and 2005, respectively. His current research interests include the design, simulation, and analysis of algorithms in the fields of computer networks, Internet of Things, vehicular networks, and big data.  

\end{IEEEbiography}

\begin{IEEEbiography}[{\includegraphics[width=1.1in,height=1.22in,clip]{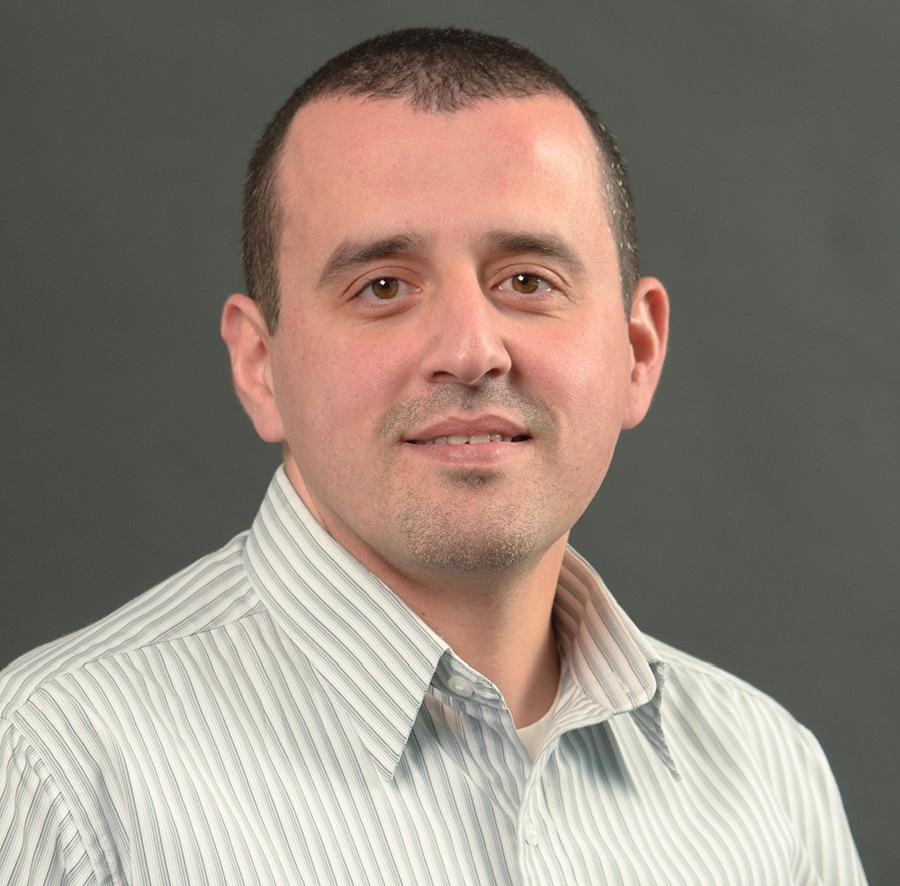}}]%
{Ala Al-Fuqaha (S'00-M'04-SM'09)}
  received Ph.D. degree in Computer Engineering and Networking from the University of Missouri-Kansas City, Kansas City, MO, USA, in 2004. He is currently a professor at Hamad Bin Khalifa University (HBKU) and Western Michigan University. His research interests include the use of machine learning in general and deep learning in particular in support of the data-driven and self-driven management of large-scale deployments of IoT and smart city infrastructure and services, Wireless Vehicular Networks (VANETs), cooperation and spectrum access etiquette in cognitive radio networks, and management and planning of software defined networks (SDN). He is a senior member of the IEEE and an ABET Program Evaluator (PEV). He serves on editorial boards of multiple journals including IEEE Communications Letter and IEEE Network Magazine. He also served as chair, co-chair, and technical program committee member of multiple international conferences including IEEE VTC, IEEE Globecom, IEEE ICC, and IWCMC.
 
\end{IEEEbiography}

\begin{IEEEbiography}[{\includegraphics[width=1.1in,height=1.22in,clip]{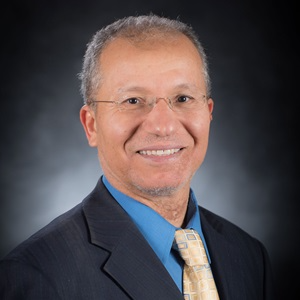}}]%
{Mohsen Guizani (S'85-M'89-SM'99-F'09)}
received the B.S. and M.S. degrees in electrical engineering and the M.S. and Ph.D. degrees in computer engineering from Syracuse University, in 1984, 1986, 1987, and 1990, respectively. He was the Associate Vice President of Qatar University, the Chair of the Computer Science Department, Western Michigan University, the Chair of the Computer Science Department, University of West Florida, and the Director of graduate studies at the University of Missouri–Columbia. He is currently a Professor with the Department of Computer Science and Engineering, Qatar University. He has authored or coauthored nine books and publications in refereed journals and conferences. His research interests include wireless communications and mobile computing, vehicular communications, smart grid, cloud computing, and security.
\end{IEEEbiography}

\begin{IEEEbiography}[{\includegraphics[width=1.1in,height=1.22in,clip]{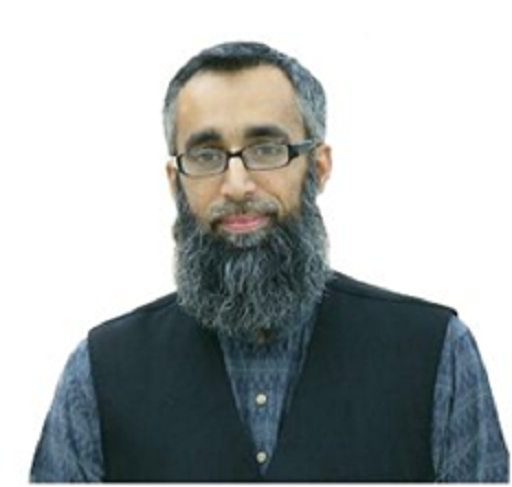}}]%
{Junaid Qadir (M'14 -- SM'14)}
received Ph.D. from University of New South Wales, Australia in 2008 and his Bachelors in Electrical Engineering from UET, Lahore, Pakistan in 2000. He is a Professor at the Information Technology University (ITU)--Punjab, Lahore. He is the Director of the IHSAN (ICTD; Human Development; Systems; Big Data Analytics; Networks Lab) Research Lab at ITU (http://ihsanlab.itu.edu.pk/). His primary research interests are in the areas of computer systems and networking and using ICT for development (ICT4D). Dr. Qadir has served on the program committee of a number of international conferences and reviews regularly for various high-quality journals. He is an Associate Editor for IEEE Access, Springer Nature Central's Big Data Analytics journal, Springer Human-Centric Computing and Information Sciences, and the IEEE Communications Magazine. He is an award-winning teacher who has been awarded the highest national teaching award in Pakistan—the higher education commission’s (HEC) best university teacher award—for the year 2012-2013. He has considerable teaching experience and a wide portfolio of taught courses in the disciplines of systems \& networking; signal processing; and wireless communications and networking. He is a senior member of IEEE and ACM. He has been appointed as an ACM Distinguished Speaker from 2020 to 2022.

\end{IEEEbiography}

\end{document}